\documentclass{siamltex}
\usepackage{eqnarray,amsmath,amsfonts,graphicx,fancybox,xcolor,url}

\newcommand{\RR}{{\mathbb{R}}}
\newcommand{\NN}{{\mathbb{N}}}
\newcommand{\CC}{{\mathbb{C}}}

\def\bA{{\mathbf A}}
\def\bB{{\mathbf B}}
\def\bC{{\mathbf C}}
\def\bD{{\mathbf D}}
\def\bE{{\mathbf E}}
\def\bF{{\mathbf F}}
\def\bH{{\mathbf H}}
\def\bI{{\mathbf I}}
\def\bL{{\mathbf L}}
\def\bM{{\mathbf M}}

\def\bQ{{\mathbf Q}}
\def\bR{{\mathbf R}}
\def\bT{{\mathbf T}}
\def\bU{{\mathbf U}}

\def\bW{{\mathbf W}}
\def\bZ{{\mathbf Z}}

\def\bc{{\mathbf c}}
\def\be{{\mathbf e}}

\def\br{{\mathbf r}}
\def\bu{{\mathbf u}}

\def\bw{{\mathbf w}}
\def\bx{{\mathbf x}}

\def\bz{{\widetilde{\mathbf z}}}

\def\balpha{{\boldsymbol{\alpha}}}
\def\bbeta{{\boldsymbol{\beta}}}
\def\bSigma{{\boldsymbol{\Sigma}}}

\newcommand{\one}{{\mathbf{1}}}

\def\tA{{\widetilde {\mathbf A}}}
\def\tB{{\widetilde {\bf B}}}
\def\tD{{\widetilde \bD}}
\def\tG{{\widetilde G}}
\def\tL{{\widetilde \bL}}

\def\tZ{{\widetilde {\mathbf Z}}}

\def\ts{{\widetilde {\bf s}}}

\def\tu{{\widetilde u}}
\def\tbu{{\widetilde \bu}}
\def\tlambda{{\widetilde \lambda}}
\def\tr{{\widetilde {\bf r}}}

\def\K{{\cal K}}
\def\hL{{\cal L}}
\def\cN{{\cal N}}

\newtheorem{remark}[theorem]{Remark}

\newtheorem{assumption}[theorem]{Assumption}
\newtheorem{algorithm}[theorem] {Algorithm}
\newtheorem{prop}[theorem] {Proposition}
\newtheorem{cor}[theorem]{Corollary 
}
\graphicspath{{./FIGS/}}

\begin{document}
\bibliographystyle{plain}

\pagestyle{myheadings}


\title{Distance preserving model order reduction of graph-Laplacians and cluster analysis}
\author{
Vladimir Druskin\footnotemark[1]
\and
Alexander V. Mamonov\footnotemark[2]
\and
Mikhail Zaslavsky\footnotemark[3]
}

\renewcommand{\thefootnote}{\fnsymbol{footnote}}

\footnotetext[1]{Worcester Polytechnic Institute, Department of Mathematical Sciences,
Stratton Hall,
100 Institute Road, Worcester MA, 01609 (vdruskin1@gmail.com)}
\footnotetext[2]{Department of Mathematics, University of Houston, 
3551 Cullen Blvd., Houston, TX 77204-3008 (mamonov@math.uh.edu)}
\footnotetext[3]{Schlumberger-Doll Research Center, 1 Hampshire St., 
Cambridge, MA 02139-1578 (mzaslavsky@slb.com)}
 
\maketitle

\begin{abstract} 
Graph-Laplacians and their spectral embeddings play an important role in multiple 
areas of machine learning. This paper is focused on graph-Laplacian dimension reduction 
for the spectral clustering of data as a primary application, however, it can also be applied 
in data mining, data manifold learning, etc. 
Spectral embedding provides a low-dimensional parametrization of the data manifold 
which makes the subsequent task (e.g., clustering with k-means or any of its 
approximations) much easier. However, despite reducing the dimensionality of data, 
the overall computational cost may still be prohibitive for large data sets due to two factors. 
First, computing the partial eigendecomposition of the graph-Laplacian typically requires 
a large Krylov subspace. 
Second, after the spectral embedding is complete, one still has to operate 
with the same number of data points, which may ruin the efficiency of the approach. 
For example, clustering of the embedded data is typically performed with various 
relaxations of k-means which computational cost scales poorly with respect to the 
size of data set. Also, they become prone to getting stuck in local minima, so their 
robustness depends on the choice of initial guess.
In this work, we switch the focus from the entire data set to a subset of graph vertices 
(target subset). We develop two novel algorithms for such low-dimensional representation 
of the original graph that preserves important global distances between the nodes of the 
target subset. In particular, it allows to ensure that target subset clustering is consistent 
with the spectral clustering of the full data set if one would perform such. 
That is achieved by a properly parametrized reduced-order model (ROM) of the 
graph-Laplacian that approximates accurately the diffusion transfer function of the original 
graph for inputs and outputs restricted to the target subset.
Working with a small target subset reduces greatly the required dimension of Krylov 
subspace and allows to exploit the conventional algorithms (like approximations of k-means) 
in the regimes when they are most robust and efficient. This was verified in the numerical 
clustering experiments with both synthetic and real data. We also note that our ROM approach 
can be applied in a purely transfer-function-data-driven way, so it becomes the only feasible 
option for extremely large graphs that are not directly accessible.
There are several uses for our algorithms.
First, they can be employed on their own for representative subset clustering in cases 
when handling the full graph is either infeasible or simply not required. 
Second, they may be used for quality control. 
Third, as they drastically reduce the problem size, they enable the application
of more sophisticated algorithms for the task under consideration 
(like more powerful approximations of k-means based on semi-definite
programming (SDP) instead of the conventional Lloyd's algorithm).
Finally, they can be used as building blocks of a multi-level divide-and-conquer type 
algorithm to handle the full graph. The latter will be reported in a separate article.
\end{abstract}

\section{Introduction}
\label{sec:intro}

Embedding via graph-Laplacian eigenmaps (a.k.a. spectral embedding) is employed heavily in  
unsupervised machine learning and data science. Its power comes from the dimensionality 
reduction property, i.e., unveiling the low-dimensional manifold structure within the data.
Subsequently, this structure can be exploited efficiently by clustering algorithms.
Such strategy is commonly known as spectral clustering, e.g., 
\cite{Shi1997NormalizedCA,  Ng2001OnSC,Belkin2003LaplacianEF}.
On the downside, spectral clustering becomes computationally expensive for large graphs, 
due to the cost of computing the spectral data. The use of combinatorial clustering algorithms,
such as the true k-means (known to be NP-hard), on the spectral data can further increase 
the cost, as typically those do not scale well with graph size. 
Alternatively, one may apply heuristic approximations of k-means such as Lloyd's algorithm 
to the spectral data. In this case robustness may become an issue since the heuristics are
more likely to get stuck in local minima for larger graphs. 
Here we should also point out a recent promising approach that substitutes k-means 
by a direct algorithm \cite{Damle2016RobustAE}.

Spectral embedding reduces the dimension of the data set space, but typically does not change the 
number of data points. In principle, reducing both can be achieved via graph contraction and model 
order reduction methods, see, e.g., \cite{Cheng2016GraphSM}. One can view the 
graph-Laplacian as a second order finite-difference discretization of Laplace-Beltrami operator 
that governs the diffusion on a Riemannian data manifold. The graph contraction is conceptually 
similar to multigrid coarsening, as it constructs a coarse grain graph-Laplacian based on local information. 
As such, it can be at most a low order 
approximation of the original full-scale graph-Laplacian with respect to the maximum edge weight.
On the other hand, well developed tools of model order reduction for linear time-invariant (LTI) 
dynamical systems produce exponential convergence with respect to the order of the reduced 
order model (ROM), see, e.g., \cite{Antoulas01asurvey, Saad2018, BAI20029}. 
We note that model order reduction was already successfully employed for finding 
nearest local clusters \cite{2017-ecml-pkdd}, PageRank computations \cite{2015-edgeppr} 
and dynamics of quantum graphs \cite{doi:10.1093/imanum/drx029}.

With clustering application in mind, the objective of this work is to develop an efficient and 
accurate partitioning of an a priori prescribed \emph{arbitrary} subset of vertices of the full graph,
the so-called target subset. 
While a major goal is to avoid clustering the full graph, the resulting target subset partitioning 
must nevertheless be consistent with full graph spectral clustering that one may perform. 
That is, the partitioning must separate the vertices of the target subset that would be placed
into different clusters by the full graph spectral clustering algorithm.
Reduced-order model of graph-Laplacian allows to achieve that by taking into account the structure
of the whole graph. The main motivation behind the proposed approach is to ultimately develop 
a multi-level divide-and-conquer algorithm for full graph clustering. Obviously, the method 
proposed here is a crucial building block of such an algorithm. Full graph clustering algorithm 
based on subset clustering presented here is a topic of current research and will be presented in 
a separate article. 

We note that the smallest eigenmodes used in spectral embedding correspond to late-time 
asymptotics of diffusive LTI dynamical system with the graph-Laplacian as the system matrix. 
A great success of model reduction for diffusive problems is that a multi-input/multi-output (MIMO) 
transfer function of an LTI dynamical system with moderate sizes of input and output arrays 
(here dictated by the size of the target subset) can be well approximated by a comparably 
small numbers of state variables.

In this work we employ reduced models of square MIMO LTI systems, e.g., \cite{Antoulas01asurvey, Reichel_2016}. 
These are the systems in which the input and output functionals are the same. In our case the 
input/output is defined by the indicator vectors of the vertices in the prescribed target subset. 
To obtain an accurate approximation of the diffusive LTI system response at the target subset, 
the ROM needs to have additional, ``interior'' degrees of freedom that account for diffusion 
in the rest of the graph outside the target subset. These interior degrees of freedom 
are obtained by subsequently projecting the full graph Laplacian on certain Krylov subspaces which 
is performed via two-stage block-Lanczos algorithm, as in \cite{druskin2017multiscale}.

After the ROM is constructed, in order to cluster (partition) the target subset, 
the interior degrees of freedom must be sampled. Although the clustering of internal 
degrees of freedom is an auxiliary construction that is discarded, it is essential to 
accounting for the structure of the full graph. Here we consider two algorithms for 
sampling the interior degrees of freedom. The first algorithm samples the Ritz vectors 
computed from the ROM. The second algorithm transforms the ROM realization to a 
symmetric semi-definite block-tridiagonal matrix with zero sum rows. This matrix mimics a 
finite-difference discretization of an elliptic PDE operator with Neumann boundary conditions. 
It can be viewed as a multi-dimensional manifold generalization of the so-called 
``finite-difference Gaussian quadrature rule'' or simply the optimal grid \cite{druskin1999gaussian} 
embedding the ROM realization to the space of the full graph-Laplacian. 
The interior degrees of freedom then correspond to the sampling of the data manifold at the 
nodes of that grid. Once the interior degrees of freedom are sampled, 
any conventional clustering algorithm can be applied.

The main computational cost of the proposed approach is ROM construction via Krylov subspace 
projection. However, due to adaptivity to the target subset, it requires a smaller Krylov subspace 
dimension compared to that used by the partial eigensolvers in conventional spectral 
clustering algorithms. The subspace dimension reduction is especially pronounced for graphs with a 
large number of connected components. The significant dimensionality reduction and the regular sparse 
(deflated block-tridiagonal) structure of ROM realization matrix leads to more efficient and robust 
performance of approximate k-means algorithms. Moreover, the reduced dimensionality enables the use of 
more accurate approximations of k-means like those based on semi-definite programming (SDP)
\cite{Damle2016RobustAE,peng2007approximating,7421303} that are otherwise computationally 
infeasible.

The paper is organized as follows. We introduce the target subset along with
some notation in Section~\ref{sec:problem}. In Section \ref{sec:ROGL} we first construct ROM 
as the Pad\'e approximation of the frequency domain MIMO transfer function on the target subset, 
outline a projection-based approach to its computation (with a detailed algorithm via Krylov 
subspace projection in given Appendix~\ref{app:twostage}) and derive ROM realization via 
reduced order graph-Laplacian (ROGL) of deflated block-tridiagonal structure. At the end of 
the section we show that ROM preserves commute-time and diffusion distances between 
the nodes of the target subset (with proof in Appendix~\ref{app:distpres}). 
In Section \ref{sec:clustappl} we discuss clustering application and introduce two novel ROM-based algorithms 
for performing this task. In particular, the Ritz vectors and a clustering algorithm based 
on their sampling is given in Section~\ref{sub:ritzclust}, while the algorithm based on the 
ROGL is presented in Section~\ref{sub:romgraph}. In Section~\ref{sec:numerical} we discuss
the results of the numerical experiments and examine the quality of clustering for synthetic 
and real-world data sets. We conclude with the summary and discussion of further research 
directions in Section~\ref{sec:summary}. In Appendix~\ref{app:interpret} we discuss an 
interpretation of the ROGL in terms of finite-difference Gaussian rules. 
In Appendix~\ref{app:dblanczos} we present the details of the deflated block-Lanczos 
procedure, the main building block of our model reduction algorithms. 

\section{Problem formulation and notation}
\label{sec:problem}

Both the definition and the construction of the ROGL are \emph{targeted}, i.e., 
the ROGL will be computed to best approximate certain diffusion response of a graph 
on a specifically chosen subset of its vertices. To that end, 
consider a graph with vertex set $G = \{ 1, 2, \ldots, N \}$. To construct the ROGL 
we choose a \emph{target subset} $G_m \subset G$ of graph's vertices, where
\begin{equation}
G_m = \{ i_1, i_2, \ldots, i_m \},
\end{equation} 
with the case of most interest being $m \ll N$.

The connectivity of the graph and the weights associated with graph edges are encoded in the 
\emph{graph-Laplacian} matrix $\bL \in \RR^{N \times N}$ that has the following properties: 
$\bL$ is symmetric positive-definite, the sum of entries in every of its rows is zero, and the
off-diagonal entries are non-positive. 
Note that while the latter condition could be relaxed, i.e., ROGL can be constructed 
for graphs with negative edge weights, the former two properties are essential.
Hereafter we denote matrices with bold uppercase letters
and vectors are bold lowercase, while the individual entries of matrices and vectors are
non-bold uppercase and lowercase letters, respectively.

Widely used generalizations are normalized and symmetric normalized  graph-Laplacians, 
defined respectively as $\bD^{-1} \bL$ and $\bA=\bD^{-1/2} \bL \bD^{-1/2}$, 
where $\bD\in \RR^{N\times N}$ is a positive-definite diagonal matrix.

A particular case of  $\bD$ containing
the diagonal part of $\bL$:
\begin{equation}
\label{eqn:randwalkdiag}
\bD_{RW} = \mbox{diag}(L_{11}, L_{22}, \ldots, L_{NN})
\end{equation}
relates graph-Laplacians to random walks on graph and Markov chains \cite{Chung:1997}.
It allows us to define the \emph{random-walk normalized graph-Laplacian} 
\begin{equation}\label{eqn:randwalknorm}
\bL_{RW} = \bD_{RW}^{-1} \bL,
\end{equation}
and the \emph{random walk normalized symmetric graph-Laplacian}
\begin{equation}
\bA_{RW} = \bD_{RW}^{-1/2} \bL \bD_{RW}^{-1/2}.
\label{eqn:defA}
\end{equation}

We also need the following quantities associated with the target subset. First,
denote by $\be_j \in \RR^N$ the j$^{th}$ column of the $N \times N$ identity matrix. 
Then, the columns of
\begin{equation}
\bB = [\be_{i_1}, \ldots, \be_{i_m}] \in \RR^{N \times m},
\label{eqn:bmat}
\end{equation}
are the ``indicator vectors'' of the vertices in the target subset $G_m$.
Likewise, we can also define $\bE_1 = [\be_1, \ldots, \be_m] \in \RR^{n \times m}$,
where in a slight abuse of notation $\be_j$, $j=1,\ldots,m$, are the first  $m$ columns of the $n \times n$ 
identity matrix, for some $n > m$.

\section{Reduced order graph-Laplacian}
\label{sec:ROGL}

In this section we present the definition, numerical construction and the properties of the ROGL.

Given the target subset $G_m$ we can consider the diffusion process with both
inputs and outputs (sources and receivers) restricted to it. One way to characterize
the response of such process is via the discrete-time diffusion transfer function
\begin{equation}
\bF(p) = \bB^T\bD (\bI - \tau \bD^{-1}\bL )^p \bB = 
\bB^T\bD^{1/2} (\bI - \tau\bA )^p\bD^{1/2} \bB \in \RR^{m \times m}, \quad p=1,2,\ldots
\label{eqn:transfer}
\end{equation}
that has both its inputs and outputs in the target subset $G_m$ and $p$ is the 
discrete time with 
\begin{equation}
\label{eqn:cfl}
\tau\le \frac {2}{\|\bA\|}.
\end{equation}
Normally $\tau$ is chosen as the upper bound of (\ref{eqn:cfl}).  
An important property of the random walk graph Laplacians is  that  $\|\bA_{RW}\|\le 2$, 
so for  that case  (\ref{eqn:transfer}) can be stably rewritten with $\tau=1$ as
\begin{equation}
\bF(p) = \bB^T\bD_{RW} (\bI -  \bL_{RW} )^p \bB , \quad p=1,2,\ldots.
\label{eqn:transferRW}
\end{equation}
It is easy to see, that matrix $\bI -  \bL_{RW}$ is the matrix of transitional probabilities, 
so evolution of the transfer function given by (\ref{eqn:transferRW}) is defined by the corresponding 
Markov chain and the rescaled components of the matrix $\bF(p)$ are directly related to  
the diffusion time distance (for details see \cite{Coifman_Lafon_2006}, 
see also Section~\ref{sub:ctdiffdist}).

To define the ROGL we follow the conventional reasoning of model order reduction. 
That is, we build a small proxy model for (\ref{eqn:transfer}) of the same form
\begin{equation}
\widetilde{\bF}(p) = \tB^T \tD (\bI - \tau \tD^{-1}\tL )^p \tB \in \RR^{m \times m}, 
\quad p=1,2,\ldots
\label{eqn:FROM}
\end{equation}
so that the reduced-order transfer function (\ref{eqn:FROM}) approximates the original one, i.e.,
\begin{equation}
\widetilde{\bF}(p)\approx {\bF}(p)\label{eqn:approx}.\end{equation} 
Here $\tD^{-1},  \tL \in \RR^{n \times n}$ and $\tB\in \RR^{n \times m}$ 
{with $n \ll N$}. The ROGL $\tL$ is constructed so that it governs the diffusion LTI 
system with MIMO transfer function given by $\widetilde{\bF}(p)$.
Note that it is also possible to construct ROGL enforcing continuum-time diffusion
transfer function approximation instead of (\ref{eqn:FROM}). However, discrete-time formulation is 
more consistent with the random walk, as it proceeds one graph vertex at a time.

While enforcing an accurate approximation (\ref{eqn:approx}) is a common practice in 
model order reduction, it must be justified in the context of data science applications, 
e.g., clustering. To that end we observe that the entries of $\bF(p)$ are proportional to the 
components of diffusion distances on $G$ between the corresponding nodes of the 
target subset $G_m$, see, e.g., \cite{Coifman_Lafon_2006}. Thus, (\ref{eqn:approx}) 
implies that the ROGL respects those distances, i.e., it approximates the random walk 
between the nodes of the target set for the case of a random walk graph Laplacian.
Below we construct the ROGL satisfying the outlined above properties that yields 
an approximation (\ref{eqn:approx}) that is exponentially accurate in $n$.

\subsection{ROM via rational approximation}

Let $\left( \lambda_j\in\RR;\bu_j\in\RR^N \right)^N_{j=1}$ be the eigenpairs of 
$\bA = \bD^{-1/2} \bL \bD^{-1/2}$, i.e., 
\begin{equation}
\bA \bu_j=\lambda_j\bu_j, \quad \bu_j^T\bu_j=1, \quad j=1,\ldots,N.
\end{equation}
Hence, using the spectral decomposition 
\begin{equation}
\bD^{-1/2} \bA \bu_j = \bD^{-1}\bL\bD^{-1/2}\bu_j=\lambda_j\bD^{-1/2}\bu_j
\end{equation}
of the normalized graph-Laplacian $\bD^{-1}\bL$, we obtain
\begin{equation}
{\bF}(p) = \sum_{j=1}^N(1 - \tau \lambda_j)^p\br_j \br_j^T  , \quad p=1,2,\ldots,
\label{eqn:FROMsd}
\end{equation}
where $\br_j = \bB^T \bD^{1/2} \bu_j$.

Likewise, let $\left( \tlambda_j \in \RR; \tbu_j \in \RR^n \right)^n_{j=1}$ 
be the eigenpairs of $\tD^{-1/2} \tL \tD^{-1/2}$, i.e., 
\begin{equation}
\tD^{-1/2} \tL \tD^{-1/2} \tbu_j = \tlambda_j \tbu_j, \quad \tbu_j^T \tbu_j=1, \quad j=1,\ldots,n,
\end{equation}
Using this eigendecomposition we obtain
\begin{equation}
\widetilde{\bF}(p) = \sum_{j=1}^n(1 - \tau \tlambda_j)^p \tr_j \tr_j^T  , \quad p=1,2,\ldots,
\label{eqn:FROMsdROM}
\end{equation}
where $\tr_j = \tB^T \tD^{1/2} \tbu_j$.

Reduced order models are typically computed in the Laplace domain,
where the transfer function and its reduced order model can be respectively 
written in matrix partial fraction form as
\begin{equation}
\hL{\bF}(\lambda) = \bB^T\bD (\lambda\bI -  \bD^{-1}\bL )^{-1} \bB = 
\sum_{j=1}^N \frac{\br_j \br_j^T}{ \lambda-\lambda_j}
\label{eqn:FROMsdL}
\end{equation}
and
\begin{equation}
\hL\widetilde{\bF}(\lambda) = 
\sum_{j=1}^n\frac{\tr_j \tr_j^T}{ \lambda-\tlambda_j}.
\label{eqn:FROMsdLROM}
\end{equation}
In this framework, ROM $\hL\widetilde{\bF}(\lambda)$ can be considered as a 
rational approximation of $\hL{\bF}(\lambda)$. 

The multiplicity of the pole $\lambda=0$ of ${\hL \bF}$ provides critical information 
about the connectivity of the target set, i.e., this number (obviously not smaller than one) 
is equal to the number of connected components supported on $G_m$. 
So, the ROM must capture the restriction of the nullspace onto these components.
Thus, we require that the following holds.

\begin{assumption}
The multiplicity $m_0$ of the pole $\lambda=0$ (the number of $\tlambda_j=0$) of the 
ROM $\widetilde{\hL \bF}$ is the same as the multiplicity of the pole $\lambda=0$ 
of the exact transfer function ${\hL \bF}$ (the number of $\lambda_j=0$).
\label{ass:null}
\end{assumption}

There is a great variety of efficient model order reduction algorithms for LTI dynamical 
systems, that  produce exponential convergence with respect to the ROM's order, 
e.g., see \cite{Antoulas01asurvey, Saad2018, BAI20029,BEATTIE20155}. 

One such approach is to compute the ROM in a data-driven manner by directly approximating 
the transfer function. This approach is also known as non-invasive, i.e., it does not require an 
explicit knowledge of exact operator $\tD^{-1}\tL$ outside the target subset \cite{BEATTIE20155}. 
An alternative approach is based on projection of $\tD^{-1}\tL$ onto appropriately chosen 
subspaces, e.g., Krylov or rational Krylov subspaces. Often these two approaches produce 
algebraically equivalent results due to Pad'e-Lanczos connection \cite{BAI20029}. 
The data-driven approach can be the only feasible option when the entire graph is not available 
or it is too big to handle (in which case the transfer function can still be computed via a 
Monte-Carlo algorithm), and the projection approach is preferable for middle size graphs.

\subsection{Matrix Pad\'e approximation via Krylov subspace projection}

To construct the interpolant (\ref{eqn:FROMsdROM})
we match the first $m_0$ terms of (\ref{eqn:FROMsd}) exactly, 
as required by the Assumption~\ref{ass:null}, while the remainder 
(with $\lambda_j>0$ for $j>m_0$) is computed via Pad\'e approximation at $0$, i.e.,
\begin{eqnarray}
\tlambda_j= \lambda_j = 0, \ \tr_j=\br_j, \qquad j=1,\ldots,m_0, & \nonumber \\
\frac{d^i}{d\lambda^i}
\left. \left( \sum_{j=m_0+1}^n \frac{\tr_j \tr_j^T}{ \lambda - \tlambda_j} - 
\sum_{j=m_0+1}^{N} \frac{\br_j \br_j^T}{ \lambda - \lambda_j}\right)\right|_{\lambda=0}=0, & 
\label{eqn:pade} \\
i=0,\ldots, 2\frac{n-m_0}{m}-1, & \nonumber
\end{eqnarray}
requiring  $n-m_0=k_2m$ with an integer $k_2$. It is  known, that Pad\'e approximations 
of matrix-valued Stieltjes functions (with $O(\lambda^{-1})$ decay at infinity) are also 
Stieltjes matrix-valued functions \cite{dyukarev2004indeterminacy} with the measure 
within support of the exact functions, so $\tr_j$ are real and $\tlambda_j$ are real positive 
for $j>m_0$. It is also known that such approximations converge exponentially on 
$\CC\setminus \RR_+$ (uniformly with respect to the number of points of increase of Stieltjes 
measure $N$) with the best convergence rate of  the approximant as well as its spectral 
measure near the origin \cite{Baker1996PadAS}. This leads to uniform (with respect to $N$) 
convergence of the time-domain solution (\ref{eqn:FROMsdROM}) with the best rates for 
large diffusion times. 

Spectral embedding is normally performed for the lower part of the spectral interval 
(corresponding to large diffusion times), which is the main reason for choosing the matching 
point at $\lambda=0$.\footnote{In future work we plan to investigate the use of more powerful 
tangential multi-point Pa\'de approximations, e. g., see \cite{BEATTIE20155,druskin2014adaptive}.}  
We also note that the computation of (\ref{eqn:pade}) can be done both via data-driven 
and projection approaches. 

Recall that $m_0$ is also the number of connected components of $G$ having non-empty 
intersection with $G_m$. Denote by $\bZ \in \RR^{N \times m_0}$ the matrix with columns 
forming an orthonormal basis for the nullspace of $\bA$ within these components: 
\begin{equation}
\bA \bZ=0, \quad \bZ^T\bZ = \bI, \quad \mbox{rank}(\bB^T\bZ) = m_0.
\end{equation} 
Then, denote $\bB_\perp = (\bI-\bZ\bZ^T) \bB$ and let the columns of 
$\bQ \in \RR^{N\times (n-m_0)}$ be an orthonormal basis for the block Krylov subspace 
\begin{equation}
\K_{k_2}[\bA^{-1}, \bB_\perp]=
\mbox{colspan} \{ \bB_\perp, \bA^{-1} \bB_\perp, \ldots, \bA^{-k_2+1} \bB_\perp \}.
\end{equation} 
We should note that  
\begin{equation}
n - m_0 = \mbox{dim}\{ \K_{k_2}[\bA^{-1}, \bB_\perp] \}
\end{equation}
can be smaller than the number of columns $mk_2$ {of $\bQ$} due to irregular graph structure 
(small connected components, etc.), i.e., in general it satisfies the non-strict inequality
\begin{equation}
n-m_0 \le mk_2\label{eqn:deflation}.
\end{equation} 
The following proposition follows from a well-known connection between Pad\'e  
interpolation and Krylov subspace projection, a.k.a. Pad\'e-Lanczos connection, 
e.g., \cite {BAI20029}.
\begin{prop}\label{prop:LanczPade}
Let $\tA_{12}=\bQ_{12}^T\bA\bQ_{12}$, $\bQ_{12}= [\bQ, \bZ] \in \RR^{N\times n}$. 
Then $\hL \widetilde{\bF}(\lambda)$ from (\ref{eqn:pade}) can be equivalently obtained as
\begin{equation}
\hL\widetilde{\bF}(\lambda) = 
(\bQ_{12}^T\bD^{1/2}\bB)^T(\lambda \bI - \tA_{12})^{-1}\bQ_{12}^T\bD^{1/2}\bB.
\label{eqn:ROMPadeProject}
\end{equation}
\end{prop}

Let  $\left( \tlambda_j\in\RR;\ts_j\in\RR^n \right)^n_{j=1}$ be the eigenpairs of $\tA_{12}$. 
By construction of the projection subspace and min-max propertiues of Ritz values, 
$\tlambda _j=0$ for $j\le m_0$ and positive otherwise.  
Then (\ref{eqn:ROMPadeProject}) can be written explicitly in the form (\ref{eqn:pade}) with 
\begin{equation}
\tr_j = \bB^T\bD^{1/2}\bw_j,
\end{equation}
where $\bw_j = \bQ_{12} \ts_j \in \RR^N$ are the Ritz vectors on the column space of $\bQ_{12}$.
With the help of these vectors the state space ROM can be written in the Laplace 
and discrete time domains as
\begin{equation}
\bD(\lambda \bI - \bD^{-1}\bL )^{-1} \bB \approx
\sum_{j=1}^{m_0} \frac{\bD^{1/2}\bw_j \tr_j^T}{\lambda} + 
\sum_{j=m_0+1}^n \frac{\bD^{1/2}\bw_j \tr_j^T}{ \lambda-\tlambda_j},
\end{equation}
and 
\begin{equation}
\bD(\bI - \tau \bD^{-1}\bL )^p \bB \approx 
\sum_{j=1}^{m_0} \bD^{1/2}\bw_j \tr_j^T + 
\sum_{j=m_0+1}^n (1 - \tau \tlambda_j )^p \bD^{1/2}\bw_j \tr_j^T,
\label{eqn:state}
\end{equation}
respectively.
Again, due to min-max properties of Ritz values, (\ref{eqn:state}) is stable. 
It provides an embedding of the ROM into the state space that can further be used in 
one of our clustering algorithms. The algorithm for computation of 
$\tr_j, \tlambda_j$ and $\bw_j$ is given in Appendix~\ref{app:twostage}.

\subsection{Transformation to the graph-Laplacian form}
\label{sub:transgraph}

At this point our proposed approach diverges from the conventional model order reduction 
reasoning. While in model reduction the main quality measure is the smallness of the 
approximation error of (\ref{eqn:approx}),  here we require additionally that $\tL$ retains 
some of the key properties of the original full-scale graph-Laplacian $\bL$. 
In particular, $\tL$ must be symmetric and positive semidefinite, similarly to $\bL$, 
with the sum of elements in each row equal to zero.
  
As the first step {towards building $\tL$ with the desired properties}, 
we compute a symmetric deflated block-tridiagonal matrix 
$\tA \in \RR^{n\times n}$ such that 
\begin{equation}
\hL\widetilde{\bF}(\lambda) = 
\sum_{j=1}^n \frac{\tr_j \tr_j^T}{\lambda - \tlambda_j} \equiv 
\widehat\bD^{1/2} \bE_1^T (\lambda \bI - \tA)^{-1} \bE_1 \widehat\bD^{1/2},
\label{eqn:match}
\end{equation}
where $\widehat\bD = \bB^T \bD \bB$.

This computation can be viewed as an orthogonal transform 
$\widetilde{\bQ}^T \tA_{12} \widetilde{\bQ} = \tA$ of $\tA_{12}$ to a symmetrized normalized 
graph-Laplacian $\tA$ with deflated block-tridiagonal structure.
It is performed using the deflated block-Lanczos algorithm from Appendix~\ref{app:dblanczos} 
with inputs $\bM=\tA_{12} \in \RR^{n\times n}$ and
$\bC=\bQ_{12}^T\bD^{1/2}\bB\widehat\bD^{-1/2}$ (with orthonormal columns) 
and the output being the desired $\tA=\widetilde{\bT}$ and $\widetilde{\bQ}$ such that 
\begin{equation}
\label{eqn:initcondq}
\widetilde{\bQ}\bE_1=\bC=\bQ_{12}^T\bD^{1/2}\bB\widehat{\bD}^{-1/2}.
\end{equation}

Compared to the two preceding Lanczos processes described in Appendix~\ref{app:twostage}, 
here deflation is performed at the level of machine precision and allows to equivalently transform 
the full matrix $\tA_{12}$ to $\tA$ with block-tridiagonal structure without reducing its size.
The final step towards constructing the ROGL $\tL$ is to equivalently transform the symmetrized 
matrix $\tA$ to normalized form $\tD^{-1} \tL$ so that 
\begin{equation}
\tL = \tD^{1/2} \tA \tD^{1/2},
\end{equation}
where $\tD$ is a diagonal matrix with positive elements and $\tL$ is symmetric positive-semidefinite 
with the zero sum of elements in every row. We note that such transform is obviously non-unique.  

To enforce the row zero sum condition, consider the following.
Let the columns of an orthogonal matrix $\tZ \in \RR^{n\times m_0}$ 
be the basis vectors for the nullspace of $\tA$. 
We seek $\bz_0 \in \mbox{colspan}\{\tZ\}$ satisfying 
\begin{equation}
\label{eqn:eqforbz0}
\bB^T\bD^{1/2}\one=\bB^T\bQ_{12}\widetilde{\bQ}\bz_0.
\end{equation}
Since $\bD^{1/2}\one$ belongs to the nullspace of $\bA$ and 
$\mbox{colspan}\{ \bZ \}$ is a subspace of $\mbox{colspan}\{ \bQ_{12}\widetilde{\bQ} \}$,
we obtain that 
\begin{equation}
\label{eqn:z0}
\bz_0=\widetilde{\bQ}^T\bQ_{12}^T\bD^{1/2}\one.
\end{equation}
Using (\ref{eqn:initcondq}) and the fact that columns of $\bD^{1/2}\bB$ belong to 
$\mbox{colspan}\{\bQ_{12}\}$ we obtain 
\begin{align}
\bE_1^T\bz_0 & = \bE_1^T\widetilde{\bQ}^T\bQ_{12}^T\bD^{1/2}\one \notag\\
& =  \left(\bQ_{12}\widetilde{\bQ}\bE_1\right)^T\bD^{1/2}\one \notag\\
& =  \left(\bQ_{12}\bQ_{12}^T\bD^{1/2}\bB\widehat\bD^{-1/2}\right)^T\bD^{1/2}\one \label{eqn:bz0match}\\
& =  \left(\bD^{1/2}\bB\widehat\bD^{-1/2}\right)^T\bD^{1/2}\one \notag\\
& =  \bB^T\bD^{1/2}\one \notag\\
& =  \left(\sqrt{\bD_{i_1i_1}},\ldots ,\sqrt{\bD_{i_mi_m}}\right)^T. \notag
\end{align}
Since $\bz_0\in \mbox{colspan}\{\tZ\}$, it can be expressed as $\bz_0=\tZ\bc$ 
where the expansion coefficients $\bc \in \RR^{m_0}$ are uniquely determined from 
an $m_0 \times m_0$ linear system
\begin{equation}
(\be_1, \ldots, \be_{m_0})^T\tZ\bc=\left(\sqrt{\bD_{i_1i_1}},\ldots, \sqrt{\bD_{i_{m_0}i_{m_0}}}\right)^T.
\end{equation}
For the existence of a non-singular transform of $\tA$ to the graph-Laplacian form, 
the following assumption is required.
\begin{assumption}
All entries of $\bz_0$ defined in (\ref{eqn:z0}) are nonzero.
\label{ass:z0}
\end{assumption}

\smallskip
The validity of Assumption~\ref{ass:z0} is discussed in Appendix~\ref{app:interpret}.
Assuming that it holds along with Assumption~\ref{ass:null}, we
define a scaling of $\tA$ that transforms it to the graph-Laplacian form. 
It is summarized in the proposition below.

\begin{prop}
\label{prop:scale} 
{Let $\tA$ satisfy Assumption~\ref{ass:null}, 
and $\bz_0$ from (\ref{eqn:z0}) satisfies Assumption~\ref{ass:z0}.}
Then
\begin{equation}
\tD^{1/2} = \diag (\bz_0) \in \RR^{n \times n},
\label{eqn:Dtilde}
\end{equation} 
is a non-singular diagonal scaling matrix satisfying 
\begin{equation}
{\tD_{jj}=\bD_{i_ji_j},~j=1,\ldots,m,}
\label{eqn:diagmatch}
\end{equation}
 and
\begin{equation}
\tL = \tD^{1/2} \tA \tD^{1/2} \in \RR^{n \times n},
\label{eqn:T}
\end{equation}
is a symmetric positive semidefinite matrix with the sum of elements in every row (column) 
equal to zero, which we refer to hereafter as the reduced-order graph-Laplacian (ROGL).
\end{prop}

\begin{proof}
Note that (\ref{eqn:diagmatch}) follows directly from (\ref{eqn:bz0match}) and (\ref{eqn:Dtilde}). 
Then, since $\tA$ is symmetric positive semidefinite, (\ref{eqn:T})--(\ref{eqn:Dtilde}) 
in conjunction with Assumption~\ref{ass:z0} implies trivially that $\tL$ is symmetric and 
positive semidefinite.

From (\ref{eqn:z0}) it follows that  $ \bz_0$ is from the nullspace of $\tA$,
i.e., $\tA\bz_0=0$, so
\begin{equation}
0=\tA\bz_0= \tA \tD^{1/2} \one = \tD^{-1/2}\tL \one,
\end{equation}
then due to Assumption~\ref{ass:z0}, $\tL \one = 0$ which concludes the proof.
\end{proof}

Note that even though we enforce the zero row sum condition on $\tL$,
we do not (and cannot) impose the nonpositivity constraint on its off-diagonal entries,
which is the main distinction between the ROGL and the original graph-Lalacian $\bL$.
This can be justified by an analogy with PDE discretization. 
While $\bL$ can be viewed as second-order accurate discretization of a 
Laplace-Beltrami operator \cite{Belkin2008TowardsAT}, its reduced-order proxy $\tL$
roughly corresponds to a high-order Laplacian discretization, the off-diagonal entries of which, 
in general, may change sign. As pointed out in \cite{Knyazev2017SignedLF}, unlike the row zero sum 
property, oscillation of coefficients is not critical from the spectral embedding point of view, 
including the perspective of cluster analysis.
The corollary below and the following reasoning gives some insight into this property.
\begin{cor}\label{cor:1}
Let the ROM be computed via (\ref{eqn:pade}). Then the set of indicator vectors of connected 
components of $\tL$ forms a basis for its nullspace, i.e., the columns of $\tD^{1/2}\tZ$ are 
linear combinations of indicator vectors.
\end{cor}
\begin{proof}
First we notice that $\hL{\bF}(\lambda)$ is a block-diagonal matrix with blocks corresponding 
to the inputs and outputs in the  connected components of  $\tL$ . Computation of (\ref{eqn:pade}) 
can be decoupled for  every block  resulting in the same block structure of $\hL\widetilde{\bF}(\lambda)$.
Thus, $\tL$ is split into decoupled blocks corresponding to decoupled blocks of $\hL\widetilde{\bF}(\lambda)$. 
Then according to Proposition~\ref{prop:scale} every block has the row zero sum property, 
so the piecewise constant vectors which are constant on the decoupled blocks, 
including blocks indicator functions, belong to the nullspace of $\tL$. 
\end{proof}

Corollary~\ref{cor:1} (as well as its proof) allows to apply the perturbation theory reasoning of the 
standard graph-based cluster analysis, e.g., see \cite {Shi1997NormalizedCA,  Ng2001OnSC, Belkin2003LaplacianEF,vonLuxburg2007,Damle2016RobustAE} to the ROGL.
Connected components of graphs are trivial clusters. Removing edges via different variants of 
clustering algorithms will create multiple connected components, 
and can be viewed as a (singular) perturbation of the eigenvectors corresponding to minimal 
positive eigenvalues, which relates them to the nullspace of those  multiple connected components 
and yields the indicator vectors of the connected components exactly. Thus, spectral clustering 
algorithms are looking for the indicators that can be best approximated by the eigenvectors 
corresponding to the smallest eigenvalues of the graph-Laplacian.
Thanks to Corollary~\ref{cor:1}, the indicator vectors of connected components of ROGL can 
generate a basis on its nullspace, so the standard perturbation arguments are applicable in this case.

The embedding properties of the ROGL are also closely related to the ones of the ROM network 
realization known as finite-difference Gaussian quadrature rules or simply optimal grids, 
as discussed in Appendix~\ref{app:interpret}. 

\begin{figure}[htb]
\centering
\includegraphics[scale=0.5]{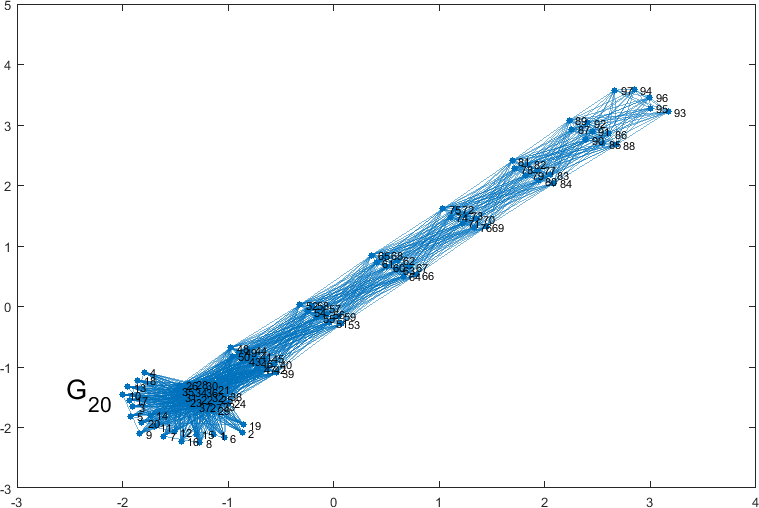} 
\caption{Reduced order graph corresponding to the rescaled deflated block-tridiagonal ROGL 
$\tL$ computed for the  Astro Physics 
collaboration network data set with $N=18872$,
$m = 20$,  $n = 97$. Vertices $1-20$ of the reduced order graph
correspond to the target set $G_{20}$ of the original graph.}
\label{fig:rog}
\end{figure}

We display in Figure~\ref{fig:rog} the reduced-order graph corresponding to ROGL $\tL$. 
It is computed for the Astro Physics collaboration network data set with $N=18872$ described 
in Section~\ref{subsection:astro}. Since the first block of $\tL$ corresponds to the original graph's vertices 
in the target subset $G_{20}$, we label them as vertices $1$ through $20$ of the reduced-order graph,
in general denoted hereafter by $\tG_m = \{ 1,2,\ldots,m \}$.
While the edge weights of the reduced order graph are strongly non-uniform 
(see the discussion in Appendix~\ref{app:interpret}), we plot all edges in the same way just to highlight the 
deflated block-tridiagonal structure of $\tL$. Deflation is observed in the shrinking of the blocks away from 
$\tG_{20}$. Overall, we observe about $200$-fold model order reduction ($n = 97$) compared to 
the original problem size ($N = 18872$) with the diffusion time distance error of the order of 
$10^{-10}$, as discussed in the following section (see also the convergence plot in 
Figure~\ref{fig:lanc_conv}).
\subsection{Commute-time and diffusion distances approximations by ROGL}
\label{sub:ctdiffdist}

Diffusion and commute-time distances are widely used measures that quantify random walks on graphs. 
We already mentioned connection of the former distance to the transfer function used in construction 
of reduced order models. Here we discuss distance-preserving properties of our ROM in more detail.

For $p \in \NN \cup \{0\}$ denote by $P^p_j(G)$ the probability clouds on graph vertices $G$, 
i.e., the probability distributions at the discrete time $p$ for a random walk originating at node $j$ 
at the discrete time $0$. The diffusion distance between two nodes at time $p$ is defined as a 
distance between probability clouds of random walks originating at those nodes 
(see \cite{Coifman_Lafon_2006}):
\begin{equation}
\left( D^p_{jk}(G) \right)^2 = \| P^p_j (G) - P^p_k(G) \|^2_{\bD}.
\end{equation}
It can be expressed in terms of normalized symmetric random-walk graph-Laplacian as 
\begin{equation}
\left(D^p_{jk}(G)\right)^2 = 
\left( \sqrt{{L_{jj}}} \be_j^T - \sqrt{L_{kk}} \be_k^T \right) 
\left( \bI - \bA_{RW} \right)^{2p}
\left( \sqrt{L_{jj}} \be_j - \sqrt{L_{kk}} \be_k \right).
\label{eqn:diffdist}
\end{equation}

Commute-time distance between two vertices is defined as the expected time it takes a random 
walk to travel from one vertex to the other and return back. Note that another metric, the 
so-called resistance distance, differs from the commute-time distance just by a constant factor. 
It can be expressed in terms of the graph-Laplacian as follows (see \cite{vonLuxburg2007}): 
\begin{equation}
C^2_{jk}(G) = \left( \be_j^T - \be_k^T \right) \bL^\dagger \left( \be_j - \be_k \right)
\end{equation}
where $\bL^\dagger$ is Moore-Penrose pseudo-inverse of $\bL$. 
Replacing $\bL$ by its normalized symmetric counterpart $\bA_{RW}$, we obtain
\begin{equation}
C^2_{jk}(G) = 
\left( \frac{1}{\sqrt{L_{jj}}} \be_j^T - \frac{1}{\sqrt{L_{kk}}} \be_k^T \right) 
\bA^\dagger_{RW}
\left( \frac{1}{\sqrt{L_{jj}}} \be_j - \frac{1}{\sqrt{L_{kk}}} \be_k \right).
\label{eqn:ctdist}
\end{equation}

Both the diffusion distance for large times $p$ and the commute-time distance take into account the 
entire structure of the graph. Hence, they represent useful tools for graph clustering compared, 
e.g., to the shortest-path distance. The structure of ROGL $\tL$ is obviously totally different than the 
structure of the original $\bL$. Therefore, even for the vertices in $G_m$ 
(that correspond to vertices $\tG_m$ in the reduced-order graph) the shortest-path distance is not 
preserved in general. However, we establish below that commute-time and diffusion distances 
between the vertices of $G_m$ are approximately preserved in the reduced-order graph with 
exponential accuracy.

\begin{prop}
\label{prop:distpres}
Let $G = \{ 1,2,\ldots,N \}$ and $\tG =\{ 1,2,\ldots,n\}$ be the vertex sets of the original graph $L_{RW}$ 
and the reduced-order graph with deflated block-tridiagonal ROGL $\tD^{-1}\tL$, respectively.
Then for any two vertices $i_j, i_k \in G_m$ that correspond to vertices $j, k \in \tG_m$, we have 
\begin{equation}
D^p_{i_j, i_k}(G) \approx D^p_{jk}(\tG) 
\mbox{ and } 
C_{i_j, i_k}(G) \approx C_{jk}(\tG)
\end{equation}
with exponential accuracy with respect to $n$.
\end{prop}

The proof is provided in Appendix \ref{app:distpres}. Note that ROGL $\tL$ may have 
positive off-diagonal elements in general. Hence, the random walk on $\tG$ can be considered 
from a quasi-probability distribution point of view, i.e., with negative probabilities 
(first introduced in \cite{Dirac_1942}) of traveling along certain edges.

\section{Applications to clustering}
\label{sec:clustappl}

Here we discuss application of the above model reduction algorithm to clustering a subset of graph's vertices.
While a major goal of this work is to avoid clustering the full graph, the resulting subset clustering 
must nevertheless be consistent with full graph clustering that one may perform. 
Assume that we choose a reference clustering algorithm, that might be too expensive to apply 
to the full graph. By consistency we understand, that if subset vertices belong to the same cluster 
of the full graph, they must to belong to the same cluster of the subset. Likewise, if they belong 
to different clusters of the full graph, they must also belong to different clusters of the subset, 
as illustrated in Figure~\ref{fig:consistency}. In some cases, however, application of a 
reference algorithm to the full graph may be infeasible, e.g., due to prohibitively large complexity, 
in which case we will use for "consistency"' comparison the knowledge of ground truth clusters,
where available.
\begin{figure}[htb]
\centering
\hspace{-0.12in}
\includegraphics[width=0.33\textwidth]{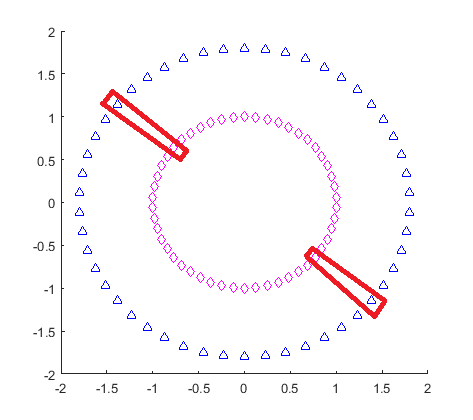}\hspace{0.001in}
\includegraphics[width=0.33\textwidth]{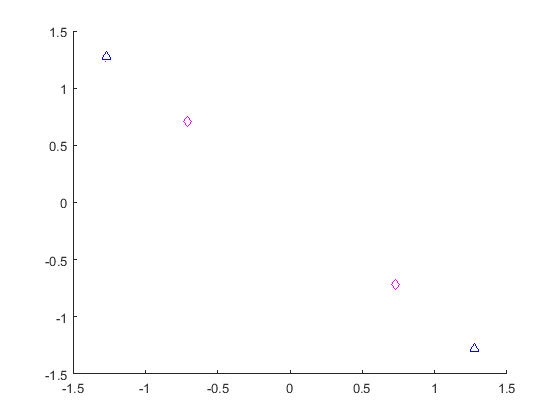}\hspace{0.001in}
\includegraphics[width=0.33\textwidth]{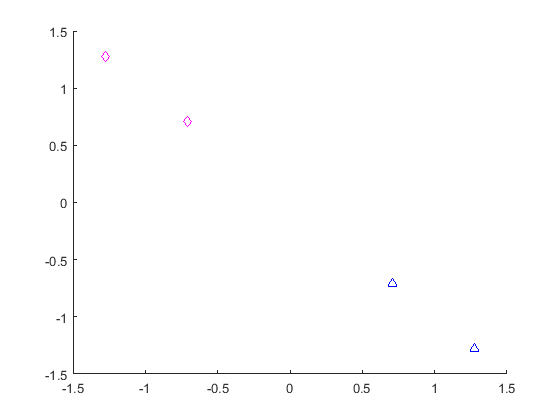}\hspace{0.001in}
\vspace{-0.12in}
\caption{Illustration of clustering consistency.
Identical markers represent data points from the same clusters. 
Leftmost plot shows two circular clusters that can be obtained using full-graph clustering. 
The target subset consisting of 4 nodes shown in red frames was clustered consistently 
by our ROM-based approaches (middle plot). 
However, without taking into account the full data set manifold, 
clustering result is typically inconsistent with the structure of the full graph (rightmost plot).}
\label{fig:consistency}
\end{figure}

It is known \cite{Belkin2003LaplacianEF,vonLuxburg2007} that one of the robust spectral clustering 
algorithms is based on the embedding via the eigenmap of the random-walk normalized graph-Laplacian. 
Thus, in this section we will use such graph-Laplacians by default, and omit $_{LW}$ subscripts for brevity.

\subsection{Embedding and clustering via Ritz vectors}
\label{sub:ritzclust}

As mentioned in Section \ref{sec:intro}, in order for the clustering of the target subset to be 
consistent with the overall graph structure, the internal degrees of freedom 
(degrees of freedom corresponding to $G \setminus G_m$) must be properly sampled. 

We will substitute the exact eigenvectors in spectral embedding by the Ritz vectors $\bw_j$.
The state space dynamics of the ROM (\ref{eqn:state}) is described by $n$ terms of spectral 
decomposition with Ritz pairs compared to $N \gg n$ terms in the full spectral decomposition 
with exact eigenpairs. Such compression is partly explained in Remark~\ref{rem:conncomp}. 
The tradeoff is that the convergence of the late time diffusion on $G_m$ 
given by (\ref{eqn:FROMsdL}) does not guarantee the same accuracy level of (\ref{eqn:state}) 
on $ G \setminus G_m$. The latter is just the square root of the former (see, e.g., \cite{Reichel_2016}) 
even though both errors decay exponentially with respect to $n$.  Moreover, spectral decomposition 
using the same Ritz pairs $\tlambda_j$, $\bw_j$ may not converge to the true solution at all, 
if one applies it when both the inputs and outputs are on $G \setminus G_m$. 
However, we assume that this should not affect clustering quality on $G_m$.  
The reasoning behind this assumption will become more clear when we discuss the 
geometric interpretation of the ROM via ROGL in Section \ref{sub:romgraph}.

To perform the spectral embedding on Ritz vectors, we need to sample them at
some vertices $q_j$ of the graph $G$. 
Let $q_j \in G$, $j = 1,2,\ldots,n_s$ be such sampling vertices, so that 
\begin{align}
q_j = i_j \in G_m \quad & \mbox{if} \quad 1 \leq j \leq m, \\
q_j  \in G \setminus G_m \quad & \mbox{if} \quad m+1 \leq j \leq n_s,
\end{align} 
i.e., the first $m$ vertices are always the ones form the target subset, 
while the remaining $n_s - m$ auxiliary vertices are chosen according to
Algorithm \ref{alg:cluster1} presented below.
Sampling the Ritz vectors at these vertices means computing the inner products 
$\be_{q_j}^T \bw_k$,
$j = 1,2,\ldots,n_s$, $k = 1,2,\ldots,n$.

\begin{algorithm}[Ritz vectors sampling clustering {[RVSC]}]\\
\label{alg:cluster1}
\noindent \textbf{Input:} normalized symmetric graph-Laplacian $\bA \in \RR^{N \times N}$, 
target subset $G_m$, number $n_c$ of clusters to construct, size $n_0$ of embedding subspace, 
set $\{q_j \in G\}^{n_s}_{j=1}$ of sampling points, numbers of Lanczos steps $k_1$, $k_2$ 
for the first and second stage deflated block-Lanczos processes in Algorithm~\ref{alg:twostagerom}, 
respectively, and   truncation tolerance $\varepsilon$. \\
\textbf{Output:} Clusters $C_1, C_2, \ldots ,C_{n_c} \subset G_m$.\\
\textbf{Step 1:} Use Algorithm \ref{alg:twostagerom} to compute  Ritz vectors 
$\bw_k$\footnote[4]{Vectors $\bw_k$ need only computed at sampling points, 
and only rows  $\be_{q_j}^T\bQ_1$  need to be stored, 
which leads to significant savings in Algorithm~\ref{alg:twostagerom}.}
 and compute the sample matrix $\bH \in \RR^{n_s\times n_0}$:
\begin{equation}
\bH_{jk} = \be_{q_j}^T \bw_k, \quad j=1,\ldots,n_s, \quad k=1,\ldots,n_0.
\end{equation} 
\noindent \textbf{Step 2:} Apply an approximate k-means algorithm, e.g., Lloyd's or SDP, 
to the matrix $\bH$ to cluster $\{q_j\}^{n_s}_{j=1}$ into $n_c$ clusters 
$\widetilde{C}_1, \widetilde{C}_2,\ldots,\widetilde{C}_{n_c}$.\\
\textbf{Step 3:} Set $C_j = \widetilde{C}_j \cap G_m$, $j=1,\ldots, n_c$.\\
\end{algorithm}

Typically, we choose $n_s$ equal to the ROM order $n$ and the sampling vertices $q_j$ are chosen 
randomly following the reasoning outlined in Appendix~\ref{app:interpret}.

\subsection{Clustering via reduced order graph-Laplacian}
\label{sub:romgraph}

In this section we present the second reduced order clustering algorithm, based on transforming the 
ROM to ROGL form. Unlike Algorithm~\ref{alg:cluster1}, no embedding into $\RR^N$ is needed for 
ROGL-based clustering. Instead, all $n \ll N$ internal degrees of freedom are sampled directly from the 
ROGL {and are treated similarly to those of} the full-scale graph-Laplacians. 
This allows to use the approach in a purely {\it data-driven fashion}, 
e.g., when {\it only diffusion information on the target set is available and the entire graph is not 
directly accessible}. This particular aspect is a part of our plan for future research.

As discussed  in Appendix~\ref{app:interpret}, the weight matrix $\tD$ in (\ref{eqn:Dtilde}) 
differs significantly from the diagonal of $\tL$ when we applied it to  full random-walk normalized graph 
Laplacians, i.e., the ROGL normalization is not the same as in the full graph case. 
This difference manifests itself as ``ghost clusters'' that do not intersect the target set, 
in particular, in the parts of the reduced order graph corresponding to last blocks of $\tL$ that have 
little influence on diffusion at $\tG_m$. This issue can in principle be resolved by reweighing 
{Lloyd's or SDP} algorithms, however, here we suggest an alternative solution
by introducing auxiliary clusters, as explained below. 

 \begin{algorithm}[Finding an optimal number of auxiliary clusters]\\
\label{alg:hak}
For a reduced-order graph clustered into $n_t$ clusters, denote by $n_g$ the number 
of such clusters that have a nonempty intersection with $\tG_m$. 
This defines $n_g = n_g(n_t)$ as a function of $n_t$. This function is piece-wise constant 
with a number of plateaus, the intervals of $n_t$ where $n_g$ assumes the same value. 
Note that in practice $n_g \ll n_t$. Next, for a desired number of clusters $n_c$, 
we choose the optimal value of $n_t$, denoted by $n_t^\star$, such that the corresponding 
$n_g^\star = n_g(n_t^\star)$ is as close to $n_c$ as possible (ideally, $n_c = n_g^\star$) 
and $n_t^\star$ is at the midpoint of the corresponding plateau.
\end{algorithm} 

Algorithm~\ref{alg:hak} ensures that the number of auxiliary clusters produces 
the number of clusters at the target subset as close as possible to the desired one 
in the most stable manner. This provides extra flexibility in finding the optimal number 
of auxiliary clusters at a cost of recomputing clustering for a number of test values of 
$n_t$, and, possibly, using different relaxations of k-means and normalized cut
on the precomputed ROGL $\tD^{-1}\tL$.
The resulting ROGL-based clustering algorithm is summarized below.
~
\begin{algorithm}[Reduced order graph-Laplacian clustering {[ROGLC]}]\\
\label{alg:cluster2}
\noindent
\textbf{Input:} normalized symmetric graph-Laplacian $\bA \in \RR^{N \times N}$, target subset $G_m$, 
number $n_c$ of clusters to compute, 
 numbers of Lanczos steps $k_1$, $k_2$ for the first and second stage deflated block-Lanczos 
processes in Algorithm~\ref{alg:twostagerom}, respectively, and truncation tolerance $\varepsilon$.  \\
\textbf{Output:} The optimal number $n_g^\star$ of clusters in $G_m$ and 
clusters $C_1, C_2, \ldots, C_{n_g^\star} \subset \tG_m$. \\
\textbf{Step 1:} Apply Algorithm~\ref{alg:twostagerom} to compute the deflated matrix
$\tA_{12} \in \RR^{n \times n}$ and the orthogonal matrix $\bQ_{12} \in \RR^{N \times n}$.\\
\textbf{Step 2:} Apply Algorithm~\ref{alg:DBL} to perform the deflated block-Lanczos process 
with input $\bM=\tA_{12}$, $\bC=\bQ_{12}^T\bD^{1/2}\bB\widehat{\bD}^{-1/2}$ with 
$\widehat\bD = \bB^T \bD \bB$  and output $\tA=\widetilde{\bT}\in\RR^{n\times n}$. \\
\textbf{Step 3:} Compute the orthogonal matrix $\tZ\in\RR^{n\times m_0}$ such that
$\mbox{colspan}\{\tZ\} = \mathcal{N}(\tA)$ and use (\ref{eqn:z0})--(\ref{eqn:T}) 
to compute the reduced-order graph-Laplacian $\tL \in \RR^{n\times n}$ 
and the diagonal normalization matrix $\tD \in \RR^{n\times n}$.\\
\textbf{Step 4:} Apply approximate k-means or normalized cut algorithm to 
$\tD^{-1}\tL \in \RR^{n\times n} $ for trial values of $n_t$ to cluster $n$ reduced-order 
graph vertices into $n_t$ clusters $\widetilde{C}_1$, $\widetilde{C}_2$, $\ldots$, $\widetilde{C}_{n_t}$, 
and find $C_j$, $j = 1, \ldots, n_g$, 
as their nonempty intersections with $\tG_m$.\\
\textbf{Step 5:} Compute the optimal numbers of clusters $n_t^\star$ and $n_g^\star$ as defined 
in Algorithm~\ref{alg:hak} and set the output clusters to the corresponding 
$C_1, C_2, \ldots ,C_{n_g^\star} \subset \tG_m$ from Step 4.
\end{algorithm}


\section{Clustering examples}
\label{sec:numerical}

We validate the performance of our clustering algorithms over three scenarios. 
The first one is a synthetic weighted 2D graph with two circular clusters. 
The other two graphs were taken from SNAP repository \cite{snapnets}: 
one with ground-truth communities (E-Mail network) and one without 
(collaboration network of arXiv Astro Physics).
For the sake of brevity, hereafter we refer to both clustering algorithms by their acronyms,
i.e., to Algorithm~\ref{alg:cluster1} as RVSC and to Algorithm~\ref{alg:cluster2} as ROGLC. 
In all the examples we used  random-walk normalized graph-Laplacian formulation  
(\ref{eqn:randwalkdiag})--(\ref{eqn:randwalknorm}). We should point out that in all 
our numerical experiments we use kmeans++ \cite{kmeanspp} implementation of Lloyd's algorithm, 
which is one of the most advanced implementations available.

\subsection{Synthetic 2D weighted graph}
\label{subsec:synth}

The first scenario we consider is a synthetic data set consisting of $N = 100$ points in the 
2D plane $\{ \bx_i \in \RR^2\}^N_{i=1}$ comprised of $2$ equally sized circular ``clouds'' of
points (we reserve the name ``clusters'' to the subsets computed by the clustering algorithms), 
as shown in the leftmost plot in Figure~\ref{fig:consistency}. Corresponding to this data set, 
we construct a weighted fully connected graph with the corresponding graph-Laplacian entries 
$L_{ij}$ defined via the heat kernel:
\begin{equation}
L_{ij} = - e^{-\| x_i-x_j \|^2 / \tau^2}, \quad i \ne j,
\label{eqn:heatker}
\end{equation}
where $\| \cdot \|$ is the standard Euclidean norm in 2D. The similarity measure and, consequently, 
graph clustering depend strongly on the choice of parameter $\tau$ in (\ref{eqn:heatker}). 
For small $\tau$ the heat kernel is close to Dirac's $\delta-$function, so the distance between 
any two vertices is close to $0$. In contrast, for large $\tau$ the distance between
all vertices is approximately the same. The most difficult and interesting case for 
clustering would be an intermediate case, i.e. when the nodes from each cloud form their separate 
clusters, but each cluster remains coupled with some of its neighbors. 
For the numerical experiments in this scenario we choose $\tau=0.6$.

We choose a target subset $G_{4}$ corresponding to two opposite points from each cloud 
(shown as red frames in the leftmost plot in Figure \ref{fig:consistency}). We expect the 
algorithm for clustering of $G_4$ to place two points from the inner circle into one cluster and 
two remaining points into another one. For our ROM-based approaches we used parameters 
$k_1=20$, $k_2=4$ for Algorithm~\ref{alg:twostagerom}. For these parameters and 
threshold $\epsilon=10^{-8}$ no deflation occurred, so the resulting Krylov subspaces have 
dimensions $n_1 = 80$ and $n = 16$, respectively.
For the clustering with both our algorithms we use the spectral embedding into the 
subspace of dimension $2$. We used the same number of clusters $n_c = 2$  
for all 3 algorithms. That  yielded  $n_t^\star = 12$ and $n_g^\star = n_c = 2$ for ROGLC.
We benchmark the clustering of $G_{4}$ against the random-walk normalized graph Laplacian 
spectral clustering (RWNSC) of the full graph, as described in \cite{vonLuxburg2007}.

The results of both our approaches coincide and they are shown in the middle plot in 
Figure~\ref{fig:consistency}. As one can observe, target subset clustering results are consistent with 
full graph clustering (RWNSC). This was achieved by constructing ROMs that take into account the topology of 
the entire data manifold. Without doing that, the results can be wrong. In particular, if we exclude all 
other $96$ vertices from the graph and remove all the adjacent edges then the clustering of just $4$ 
vertices from $G_4$ would be wrong (see the rightmost plot in Figure~\ref{fig:consistency}).

\subsection{E-Mail network with ground-truth communities}

In the second scenario we consider the graph of ``email-Eu-core'' network generated using email 
data from a large European research institution, available at the SNAP repository \cite{snapnets}. 
The original graph is directed, so we symmetrize $\bL$ for our numerical experiment. 
The network consists of $N=1005$ nodes that are split into $42$ ground-truth communities.
However, this data set is quite ``noisy'' in a sense that not all of the ground-truth communities are 
clearly recognizable from the graph structure. For example, the graph contains obvious outliers, 
i.e., isolated vertices, as shown in Figure~\ref{fig:emailgraph}. We remove them prior to clustering, 
however, some outliers are not as obvious and cannot be easily removed without affecting the 
overall graph structure. Thus, as we shall see, this graph demands particularly robust clustering 
algorithms, which are less sensitive to noise in the data.

\begin{figure}[t]
\centering
\includegraphics[scale=0.25]{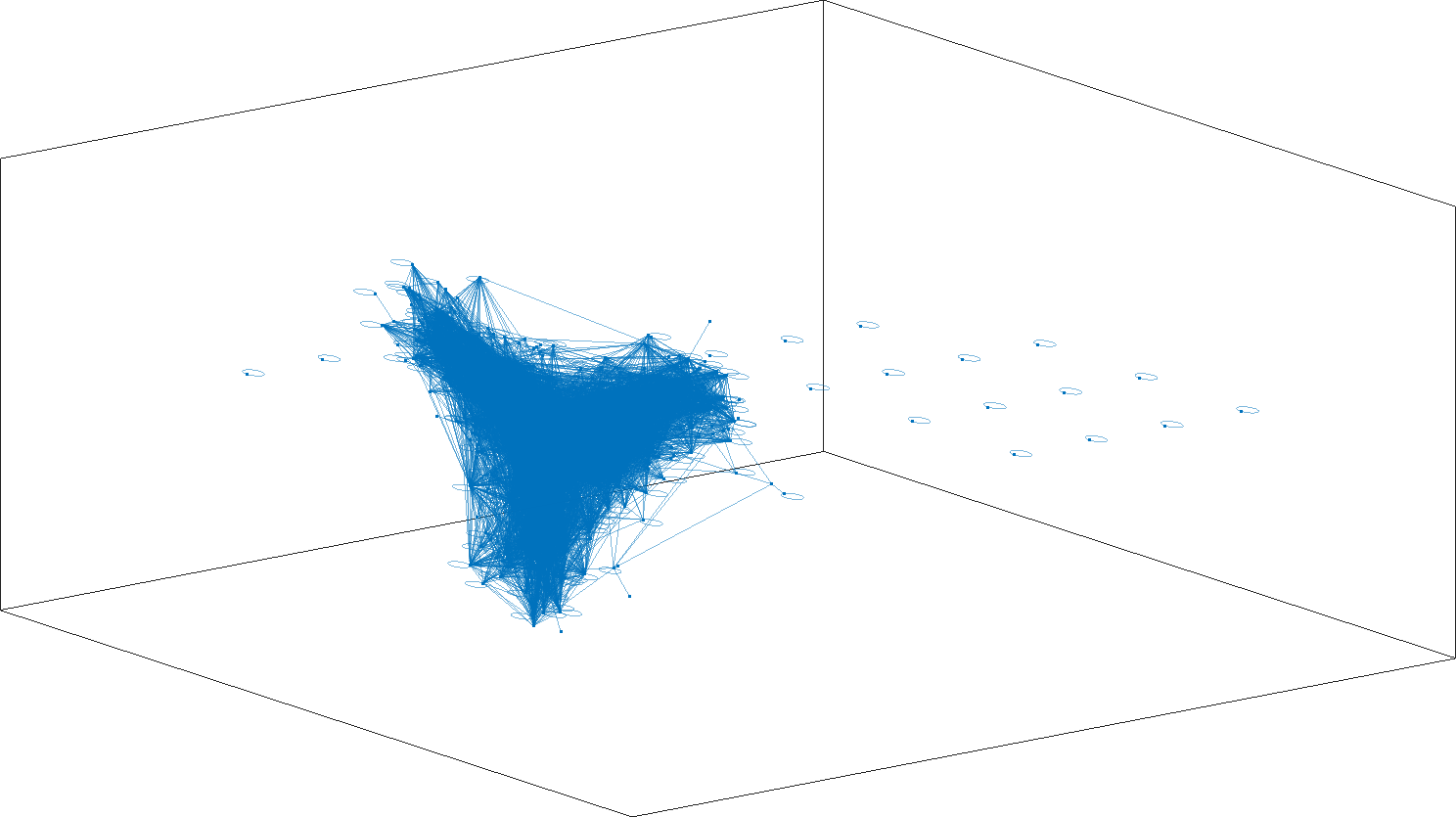} 
\caption{3D embedding of ``email-Eu-core'' network from SNAP depository. 
The graph contains a large number of self-looped outliers.}
\label{fig:emailgraph}
\end{figure}

We consider the clustering of $G_{84}$ consisting of two randomly selected vertices from each 
of the $42$ ground-truth communities. For the reference clustering with RWNSC we use the spectral 
embedding into the subspace of dimension $25$ and set $n_c = 42$. The same parameters 
were used for RVSC algorithm. For ROGLC we applied spectral embedding with the same subspace 
dimension and $n_t^\star = 170$ that corresponds to $n_g^\star=41$. The numbers of 
block-Lanczos steps for Algorithm~\ref{alg:twostagerom} were set at $k_1=10$ and $k_2=3$. 
Similarly to the previous example, for threshold $\epsilon=10^{-8}$ no deflation occurred, 
so the resulting Krylov subspace dimensions are $n_1 = 840$ and $n = 252$, respectively.

\begin{figure}[tbh]
\centering
\begin{tabular}{ccc}
RWNSC & RVSC & ROGLC \\
\hspace{-0.12in}
\includegraphics[width=0.30\textwidth]{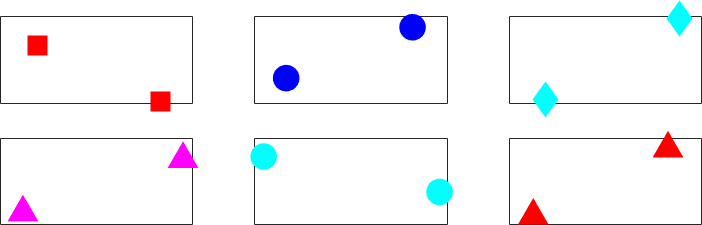} &  \hspace{0.001in}
\includegraphics[width=0.30\textwidth]{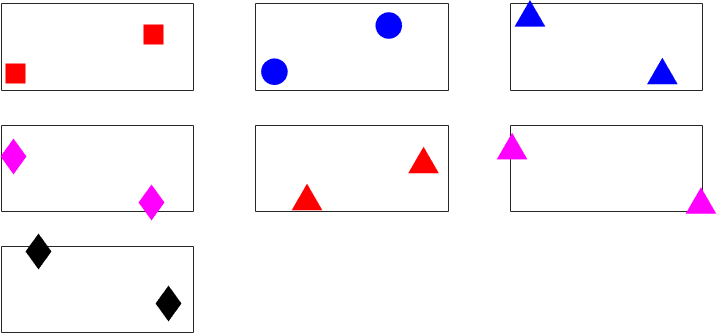}&   \hspace{0.001in}
\includegraphics[width=0.30\textwidth]{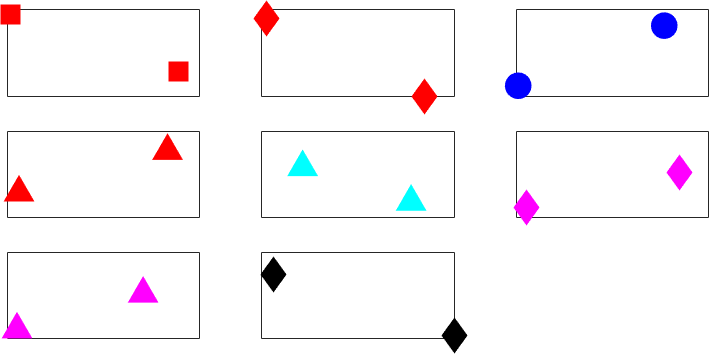}
\end{tabular}
\vspace{-0.12in}
\caption{Clusters of $G_{84}$ that are parts of some ground-truth communities: 
obtained using the full graph clustering RWNSC (left), as well as with RVSC (middle) and ROGLC 
(right) algorithms. Boxes represent different clusters. Markers of the same shape and 
color represent the vertices from the same ground-truth community. For this particular data subset, 
RWNSC clusters match the structure of $6$ ground-truth communities, while RVSC and ROGLC match 
$7$ and $8$ of them, respectively.}
\label{fig:email_res}
\end{figure}

The typical results of clustering comparing our algorithms to RWNSC are shown in Figure~\ref{fig:email_res}. 
We display the ground-truth communities in $G_{84}$ that were reproduced by the clustering algorithms, 
i.e., the corresponding vertices in $G_{84}$ were assigned to the same clusters without mixing with the 
vertices from other ground-truth communities. In this example our algorithms managed to recover slightly 
more communities than RWNSC. We ran multiple realizations with different choices of $G_{84}$, and in 
approximately half of the cases our algorithms recover the same number of communities as RWNSC, 
while in the remaining half of the cases RVSC and ROGLC recover slightly more communities, 
with the latter being the best performer. We can speculate that better performance is achieved due to 
the smaller size of data set that allows us to exploit Lloyd's algorithm in the regime where it is most 
robust and efficient.

Dimensionality reduction not only makes Lloyd's algorithm more robust, but it also allows to apply 
more accurate approximations to k-means like those based on semi-definite programming (SDP).
These algorithms typically provide more accurate clustering results, but become 
infeasible for large data sets due to the prohibitively high computational cost. In our next example 
we show the benefits of replacing Lloyd's method by SDP \cite{sdp} in ROGLC.  We should stress, 
that even for ROGLC from the previous example the SDP is still prohibitively expensive, 
therefore we had to reduce the size of $G_m$ from $m=84$ to $m=20$.
We consider $G_{20}$ consisting of two randomly selected vertices from $10$ randomly chosen 
ground-truth communities. Our reference clustering with RWNSC was the same as in the previous 
example. For ROGLC we used the numbers of block-Lanczos steps $k_1=10$ and $k_2=3$ 
which resulted in Krylov subspaces of dimension $n_1=100$ and $n=60$, respectively. 
The number of clusters was set to $n_t^\star=32$.

\begin{figure}[t]
\centering
\begin{tabular}{ccc}
RWNSC & ROGLC-L & ROGLC-S \\
\hspace{-0.12in}
\includegraphics[width=0.18\textwidth]{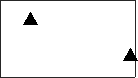} &  \hspace{0.001in}
\includegraphics[width=0.34\textwidth]{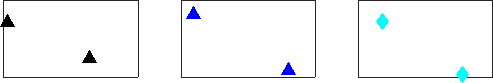}&   \hspace{0.001in}
\includegraphics[width=0.34\textwidth]{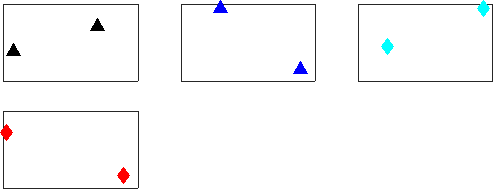}
\end{tabular}
\vspace{-0.02in}
\caption{Clusters of $G_{20}$ that are parts of some ground-truth communities: 
obtained using the full graph clustering RWNSC (left), as well as with ROGLC with Lloyd's algorithm 
(ROGLC-L, middle) and ROGLC with SDP (ROGLC-S, right). 
Boxes represent different clusters. Markers of the same shape and 
color represent the vertices from the same ground-truth community. For this particular data subset, 
RWNSC clusters match the structure of a single ground-truth community, while ROGLC-L 
matches $3$ of them. Employing ROGLC-S allowed to recover one more community, i.e., a total of $4$.}
\label{fig:email_res_10cls}
\end{figure}

\begin{figure}[t]
\centering
\includegraphics[scale=0.5]{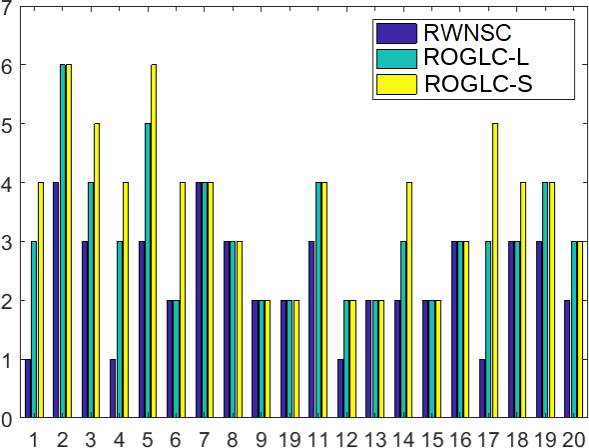} 
\caption{Comparison of full graph spectral clustering (RWNSC), ROGL clustering with 
Lloyd's algorithm (ROGLC-L) and ROGL clustering with SDP algorithm (ROGLC-S). 
Horizontal axis: random realizations of $G_{20}$;
vertical axis: number of correctly identified communities (out of $10$),
observe that it increases monotonically from RWNSC to ROGLC-L and ROGLC-S.}
\label{fig:stats-roglc}
\end{figure}

In Figure~\ref{fig:email_res_10cls} we compare the typical results of clusterings using RWNSC 
and two variants of ROGLC, with Lloyd's algorithm (ROGLC-L) and with SDP (ROGLC-S). 
Similarly to the previous example, ROGLC-L reproduced more communities compared to conventional RWNSC 
($3$ against $1$). Moreover, replacing Lloyd's algorithm with SDP allowed us to recover 
one more community. As mentioned above, the size of the target subset was reduced to 
$m=20$ that was the maximum for which applying the SDP algorithm was feasible due to 
the rapidly increasing computational cost. We ran multiple experiments with different random 
realizations of $G_{20}$ and observed that ROGLC-S always performed as well or better 
than ROGLC-L in terms of the number of correctly identified ground-truth communities, 
while the latter showed similar improvement compared to RWNSC.

Comparison results for multiple realizations of $G_{20}$ are shown in Figure~\ref{fig:stats-roglc}.
This result reinforces our observation above that ROGLC-S always performs as well or better 
compared to both RWNSC and ROGLC-L in terms of the number of correctly identified
communities. 
Most likely, this is due to the increased robustness of ROGLC algorithms with
respect to the noise in the data. As mentioned in the beginning of the section, ``email-Eu-core''
data set is rather noisy, which presents challenges to RWNSC that are overcome by both 
ROGLC variants.

\subsection{Astro Physics collaboration network}
\label{subsection:astro}

Following \cite{Damle2016RobustAE}, in the third scenario we consider ``ca-AstroPh'' network 
(available at \cite{snapnets}) representing collaborations between papers submitted to Astro Physics 
category of the e-print arXiv. This is our largest example graph with $N = 18772$ vertices and $289$ 
connected components. The relatively large graph size presents multiple challenges for clustering algorithms. 

It is crucial to ensure that different connected components are separated into different clusters. 
As was shown in \cite{Damle2016RobustAE} on the same data set, the conventional RWNSC algorithm 
does not guarantee that. A necessary condition for RWNSC algorithm to achieve this objective is to
compute the entire nullspace of the graph-Laplacian. In contrast, for RVSC and ROGLC clustering of 
the target subset $G_m$ we only need to take into account the connected components of the graph 
that have nonempty intersections with $G_m$. Hence, for small $m$ we only need to compute a small 
subspace of $\cN(\bL)$. Since the eigenmode computation in RWNSC is typically performed using Krylov 
subspaces similar to (\ref{eqn:krylov1}), RWNSC would require a significantly larger Krylov subspace 
compared to our algorithms.

Since the ground-truth clustering is not available for this scenario, we modify our testing
methodology from the one used in the previous two. 

First, we cluster the full graph into $n_c = 500$ clusters using RWNSC with the embedding 
(spectral) subspace of dimension $400$, which becomes our reference clustering. 
Here Lloyd's algorithm takes $1152$ seconds. In order to avoid difficulties in finding the connected 
components (as reported in \cite{Damle2016RobustAE}) we made a special choice of the initial 
guess for the k-means algorithm for RWNSC. 

Next, we chose $G_{18}$ of two random vertices per each of some randomly chosen $9$ 
reference clusters. The goal of this benchmark is to check whether our algorithms can 
group together the vertices of $G_{18}$ corresponding to the same reference clusters.
For both our algorithms we used spectral embedding subspace of dimension $n_0=10$. 
We took $k_1=80$ and $k_2=10$ block-Lanczos steps in Algorithm~\ref{alg:twostagerom}. 

In this example some connected components are significantly smaller than others. 
This results in early deflation of Krylov subspace (\ref{eqn:krylov1}) if $G_m$ contains vertices 
from small connected components. In particular, for the random realization of $G_{18}$ reported 
here, the resulting Krylov subspaces have dimensions $n_1=660$ and $n=146$, compared to, 
respectively, $m k_1 = 1440$ and $m k_2 = 180$.

\begin{figure}[tbh]
\centering
\begin{tabular}{cc}
RVSC & ROGLC \\
\hspace{-0.12in}
\includegraphics[width=0.40\textwidth]{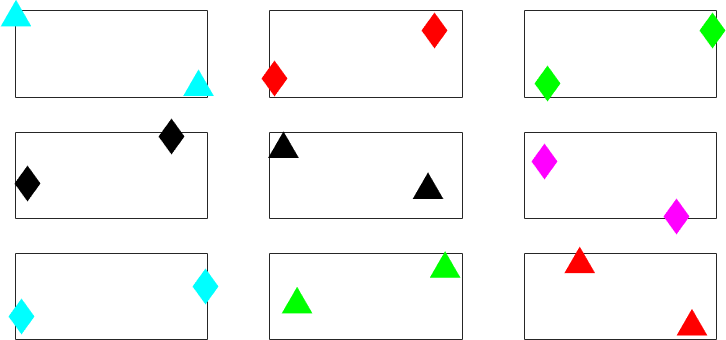} & \hspace{0.18in}
\includegraphics[width=0.40\textwidth]{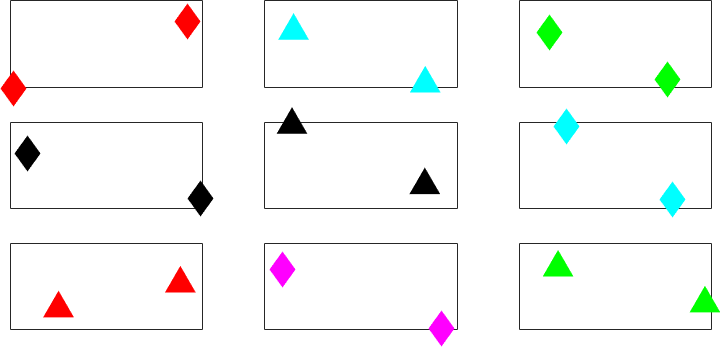}
\end{tabular}
\vspace{-0.12in}
\caption{Clusters of $G_{18}$ obtained using RVSC (left) and ROGLC (right) algorithms. 
Markers of the same shape and color represent vertices from the same reference clusters. 
Both RVSC and ROGLC successfully reproduced the reference clustering of $G_{18}$.}
\label{fig:astro_res}
\end{figure}

\begin{figure}[tbh]
\centering
\begin{tabular}{cc}
RVSC & ROGLC \\
\hspace{-0.12in}
\includegraphics[width=0.40\textwidth]{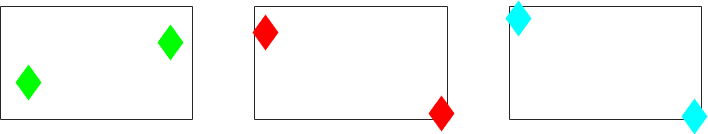} & \hspace{0.18in}
\includegraphics[width=0.40\textwidth]{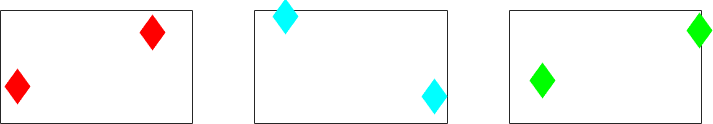}
\end{tabular}
\vspace{-0.12in}
\caption{Clusters of $G_{6}$ obtained using RVSC (left) and ROGLC (right) algorithms. 
Markers of the same shape and color represent vertices from the same reference clusters. 
Both RVSC and ROGLC successfully reproduced the reference clustering of $G_{6}$.}
\label{fig:astro_res_smallsbspc}
\end{figure}

In Figure~\ref{fig:astro_res} we display the clusters of $G_{18}$ computed using RVSC and ROGLC 
with $n_c = 9$, for which the latter produced $n_t^\star = 81$ and $n_g^\star = n_c$. We observe that 
both our algorithms have successfully reproduced all the reference clusters of $G_{18}$. Due to small 
$n_c$ in RVSC and ROGLC, Lloyd's algorithm took less than a second, compared to $1152$ seconds 
(with the same kmeans++ implementation) \cite{kmeanspp} reported above for full RWNSC clustering, 
which represents a significant computational speed-up.

Note that in this example the dimension of Krylov subspace at the first stage of 
Algorithm \ref{alg:twostagerom} is $n_1=660$. In contrast, RWNSC with ARPACK eigensolver 
\cite{lehoucq1998arpack} requires a subspace of dimension $1600$. Moreover, for smaller $m$ the 
difference between the sizes of Krylov subspaces becomes even more pronounced. For example, 
for $G_6$ of two random vertices from each of the three randomly chosen reference clusters, 
it suffices to take $n_1 = 249$ to reproduce the reference clustering, as shown in 
Figure~\ref{fig:astro_res_smallsbspc}. Algorithm~\ref{alg:twostagerom} is based on prototype 
MATLAB implementation of deflated block-Lanczos algorithm described in Appendix \ref{app:dblanczos}, 
so we can not directly compare its computational time with the highly optimized ARPACK eigensolver, 
however, it is prudent to expect that after proper implementation of the former, the speedup due to 
the reduction of the Krylov subspace dimension will be at least close to the dimension reduction factor.

\section{Summary and future work}
\label{sec:summary}

We have developed a low-dimensional embedding of a large-scale graph such that it preserves 
important global distances between vertices in an a priori chosen target subset. We presented two 
algorithms to parametrize the reduced-order graph based on Krylov subspace model order reduction. 
Both algorithms use a two-stage ROM construction procedure consisting of two block-Lanczos processes 
that compute a reduced model for the graph-Laplacian that approximates the late-time diffusion transfer 
function on the target subset. To preserve global distances between the vertices in the target subset, 
the embedding must be consistent with the overall graph structure. Hence, both algorithms need to sample 
the degrees of freedom corresponding to the vertices outside the target subset. While the first algorithm, 
``Ritz vectors sampling clustering'' (RVSC), achieves this by sampling the Ritz vectors of the ROM, 
the second one, ``reduced order graph-Laplacian clustering'' (ROGLC), transforms the ROM further to a 
form resembling a graph-Laplacian and then samples its eigenvectors.

We applied our graph embedding to target subset clustering problem. The performance of both proposed 
approaches is verified numerically on one synthetic and two real data sets against the conventional 
random-walk normalized graph Laplacian spectral clustering (RWNSC) of the full graph. 
For these scenarios our algorithms show results fully consistent with the RWNSC. In the example with 
real data containing multiple outliers (isolated sets of vertices), our algorithms produce better results 
on the target subset than RWNSC by detecting more clusters corresponding to the known ground-truth 
communities. For all the cases our algorithms show clear computational advantages, such as reduction 
of the dimension of Krylov subspaces compared to partial eigensolvers used in RWNSC.

Our graph embedding algorithm can be used standalone when the task is focused to some representative 
``skeleton'' subset and handling the full graph is not required. However, a more important application
is to use them as building blocks of multi-level divide-and-conquer type method for the full graph clustering. 
We believe that such approach has multiple advantages. First, if graph vertices are split into 
disjoint target subsets, each of them can be processed independently thus making this step 
perfectly parallelizable. 

Second, reducing dimensionality of the graph allows to exploit conventional algorithms in their 
most efficient regimes. In particular, when approximate k-means algorithms for clustering 
are applied only to the target subsets individually, we avoid the problems of their reduced robustness 
with respect to the initial guess for large data sets. Once the individual subset clusters are computed, 
one needs to merge them to obtain the clustering of the full graph. Here lies another advantage, 
since the merging step provides for greater flexibility and the possibility to control the quality of the 
clustering. The development of a multi-level divide-and-conquer method is a topic of ongoing research 
and will be reported in a separate paper. 

Third, the small size and special structure of the ROGL opens potential opportunities to substitute 
conventional approximate k-means algorithms by more accurate and expensive ones. 
For example, in clustering problems instead of using Lloyd's algorithm one may empoy SDP 
relaxations \cite{Damle2016RobustAE,peng2007approximating,7421303} that produce better 
results in certain cases, but do not scale well with the size of the problem. 

Finally, while the focus of this work is spectral clustering, the ROGL developed here has other 
potential uses. For example, one may look into adapting the min/max cut clustering algorithms 
(e.g., \cite{johnson1993min, ding2001min, flake2004graph}) to graphs with edge weights of 
arbitrary signs, like those produced by ROGL construction. Another example could be the direct 
solution of the NP-hard problem of modularity optimization \cite{Newman2016physrev} 
infeasible for large graphs, but computationally tractable for small target vertex subsets.

\section*{Acknowledgments}
\label{sec:acknowledge}

This material is based upon research supported in part by the U.S. Office of Naval Research under 
award number N00014-17-1-2057 to Mamonov. Mamonov was also partially supported by the 
National Science Foundation Grant DMS-1619821. Druskin acknowledges support by 
Druskin Algorithms and by the Air Force Office of Scientific Research under award number FA9550-20-1-0079.

The authors thank Andrew Knyazev, Alexander Lukyanov, Cyrill Muratov and Eugene Neduv for useful discussions.

\appendix

\section{Two-stage model reduction algorithm}
\label{app:twostage}

Here we introduce an algorithm for computing the orthonormal basis 
for $\K_{k_2}[{\bA^\dagger},\bB]$ defined in Proposition~\ref{prop:LanczPade} 
and all the necessary quantities for the reduced order transfer function (\ref{eqn:pade}) and 
state-space embedding (\ref{eqn:state}) (e.g., satisfying Assumption~\ref{ass:null}) 
via projection of the normalized symmetric graph-Laplacian $\bA$ consecutively on two 
Krylov subspaces. This approach follows the methodology of \cite{druskin2017multiscale} 
for multiscale model reduction for the wave propagation problem. 

At the first stage we use the deflated block-Lanczos process, Algorithm~\ref{alg:DBL}, 
to compute an orthogonal matrix $\bQ_1 \in \RR^{N \times n_1}$, the columns of which span 
the block Krylov subspace
\begin{equation}
\K_{k_1}(\bA, \bB) = \mbox{colspan} \{\bB, \bA \bB, \ldots, \bA^{k_1-1} \bB\}.
\label{eqn:krylov1}
\end{equation}
After projecting onto ${\K}_{k_1}(\bA, \bB)$, the normalized symmetric graph-Laplacian takes the 
deflated block tridiagonal form
\begin{equation}
\bT_1 = \bQ_1^T \bA \bQ_1 \in \RR^{n_1 \times n_1},
\label{eqn:project1}
\end{equation}
as detailed in Appendix \ref{app:dblanczos}. Note that the input/output matrix is transformed
simply to
\begin{equation}
\bE_1 = \bQ_1^T \bB \in \RR^{n_1\times m}.
\label{eqn:Q1TB}
\end{equation}
Observe also that $n_1 = \mbox{dim}[ \K_{k_1}(\bA, \bB) ]$, the number of columns of $\bQ_1$, 
satisfies $n_1 \leq k_1 m$ with a strict inequality in case of deflation.

\begin{remark}
\label{rem:conncomp}
Since the input/output matrix $\bB$ is supported at the vertices in the target subset $G_m$, 
repeated applications of $\bA$ cannot propagate $\bB$ outside of the connected components of the 
graph that contain $G_m$. Therefore, the support of the columns of $\bQ_1$ is included in these 
connected components. 
As a result, projection (\ref{eqn:project1}) is only sensitive to the entries of $\bA$ 
corresponding to graph vertices that can be reached from $G_m$ with a path of at most $k_1-1$ steps. 
\end{remark}

The number of block-Lanczos steps $k_1$ is chosen to attain the desired accuracy of the approximation 
of $\bZ$ (projection of $\bB$ on nullspace of $\bA$)  and the requested lower eigenmodes via $\bQ_1$ 
(i.e., to satisfy Assumption~\ref{ass:null}), that also gives good approximation of the diffusion transfer 
function on the entire time interval.

While the first stage provides a certain level of graph-Laplacian compression, the approximation 
considerations presented above may lead to the number of block-Lanczos steps $k_1$ and the resulting 
subspace dimension $n_1$ to be relatively large. It corresponds to Pad\'e approximation of the transfer 
function $\hL{\bF}$ at $\lambda=\infty$ while we are interested in $\lambda=0$ as in 
(\ref{eqn:pade}) to obtain a good approximation in the lower part of spectrum. Therefore, 
our approach includes the second stage to compress the ROM even further. This is achieved by another 
application of the deflated block-Lanczos process to construct an approximation to $\bQ_{12}$ 
from Proposition~\ref{prop:LanczPade}. 

Let the columns of matrix $\bZ_1\in\RR^{n_1\times m_0}$ form an orthonormal basis for
the nullspace of ${\bT}_1$. Then we apply the deflated block-Lanczos Algorithm from 
Appendix~\ref{app:dblanczos} to compute an orthogonal matrix 
$\bQ_2 \in \RR^{n_1 \times n}$ such that 
\begin{equation}
\mbox{colspan}(\bQ_2) = \K_{k_2}[\bT_{1}^{-1},(\bI-\bZ_1\bZ_1^T)\bE_1],
\end{equation}
where compression is achieved by choosing $k_2 < k_1$. The total dimension $n$
satisfies (\ref{eqn:deflation}) with a strict inequality in case of deflation. 

For large enough $k_1$, matrices 
\begin{equation}
\bZ \approx \bQ_1\bZ_1, \quad  
\bQ \approx \bQ_1\bQ_2, \quad 
\bQ_{12} \approx [\bQ, \bZ]
\label{eqn:12}
\end{equation}
and
\begin{equation}
\tA_{12} = \bQ_{12}^T \bA \bQ_{12} \approx [\bQ_2, \bZ_1]^T \bT_1 [\bQ_2, \bZ_1]
\label{eqn:bigT},
\end{equation}
approximate their counterparts from Proposition~\ref{prop:LanczPade} and
thus we obtain $\tlambda_j$ and  $\ts_j\in\RR^n$  as  the eigenpairs of $\bT$ and  also 
\begin{equation}
\bw_j =\bQ_{12}\ts_j.
\label{eqn:bw}
\end{equation}

We summarize the model reduction algorithm below.
\begin{algorithm}[Two-stage model reduction]\\
\label{alg:twostagerom}
\textbf{Input:} normalized symmetric graph-Laplacian $\bA \in \RR^{N \times N}$, target subset $G_m$,
numbers of Lanczos steps $k_1$, $k_2$ for the first and second stage deflated block-Lanczos 
processes,  respectively, and the truncation tolerance $\varepsilon$. \\
\textbf{Output:} $n\le N$, $m_0\le m$, $\tA_{12}$, $\ts_j$ and  $\bw_j$ for $j=1,\ldots, n$.

\noindent
\textbf{Stage 1:} Form the input/output matrix $\bB$ (\ref{eqn:bmat}) and perform the deflated
block-Lanczos process with $k_1$ steps on $\bA$ and $\bB$ with truncation tolerance $\varepsilon$, 
as described in Appendix \ref{app:dblanczos}, 
to compute the orthonormal basis $\bQ_1 \in \RR^{N \times n_1}$ for the block Krylov subspace 
(\ref{eqn:krylov1}) and the deflated block tridiagonal matrix $\bT_1$. Compute $m_0$, $\bZ_1$.\\
\textbf{Stage 2:} Perform $k_2$ steps of the deflated block-Lanczos process using matrix 
$\bT_{1}^{-1}$ and initial vector $(\bI-\bZ_1\bZ_1^T)\bE_1$, with
$\bE_1 \in \RR^{n_1 \times m}$ and truncation tolerance $\varepsilon$, as described in Appendix 
\ref{app:dblanczos}, to compute $n$ and the orthogonal matrix basis $\bQ_2 \in \RR^{n_1 \times n}$.
Compute the remaining elements of the output using (\ref{eqn:12})--(\ref {eqn:bw}).
\end{algorithm}
\begin{remark}
Due to good compression properties of Krylov subspaces, $n \ll n_1 \ll N$, thus, the computational 
cost of Algorithm \ref{alg:twostagerom} is dominated by the first stage block-Lanczos process.
In turn, assuming that no deflation occurs and each column of $\bA$ has on average 
$M$ nonzero entries, the cost of the first stage is driven by matrix products of $\bA$ and the
blocks of $\bQ_1$ (containing $m$ columns each, see step~(2a) of Algorithm~\ref{alg:DBL}). 
Since $k_1$ such products are computed, the computational cost of the first stage and of the whole 
Algorithm \ref{alg:twostagerom} can be estimated as $O(k_1 M N m)$. Note that this analysis excludes
an expensive reorthogonalization step (2j) of Algorithm~\ref{alg:DBL} that we do not perform in
Stage 1 of  Algorithm~\ref{alg:twostagerom}, as mentioned in Appendix~\ref{app:dblanczos}.
\end{remark}

To illustrate the compression properties of both stages of Algorithm~\ref {alg:twostagerom}, 
we display in Figure~\ref{fig:lanc_conv} the error of the transfer function for both Lanczos processes
corresponding to the late, nullspace dominated, part of the diffusion curve for the Astro Physics 
collaboration network data set with $N=18872$ described in Section~\ref{subsection:astro} 
with $m=20$. For the first stage we plotted dependence of the error on $k_1$ for $k_2=15$.
The second stage was performed using $\bT_1$ and $k_1 = 30$ that adds $10^{-13}$ of relative error. 
Both curves exhibit superlinear (in logarithmic scale) convergence in agreement with the bounds 
of \cite{druskin1995krylov, druskin1989two}. Even without accounting for deflation, the first stage provides more 
than $30$-fold compression of the full graph, and due to much faster convergence of $\bF_2$, 
the second stage provides more than two-fold additional compression.

\begin{figure}[htb]
\centering
\includegraphics[scale=0.45]{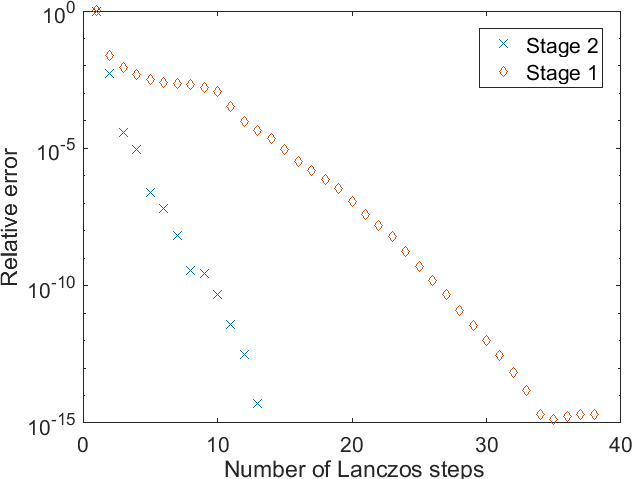} 
\caption{Relative errors of Stage 1 and Stage 2 deflated block-Lanczos processes 
in Algorithm~\ref{alg:twostagerom} versus the numbers of Lanczos steps $k_1$ and $k_2$, 
respectively, for AstroPhysics collaboration network data set.}
\label{fig:lanc_conv}
\end{figure}

The choice of parameters $k_1$ and $k_2$ depends strongly on the graph structure,  
and normally is made adaptively by a posteriori error control, e.g., by extrapolating error from three
consecutive iterations. For some scenarios we do not even need to perform Stage 2 of 
Algorithm \ref{alg:twostagerom}. For example, let us consider a family of graph Laplacians 
$\bL \in \RR^{N \times N}$ with random entries $L_{ij} \in \{0; -1\}, ~ i \ne j$, 
chosen with probability 
\begin{equation}
p \left( \{ L_{ij} = -1 \} \right) = 0.01, \quad i \neq j, \quad i,j = 1,\ldots,N,
\label{eqn:probLij}
\end{equation}
where $N=3000$, $6000$, $12000$, $24000$ and $48000$.

\begin{figure}[htb]
\centering
\includegraphics[scale=0.35]{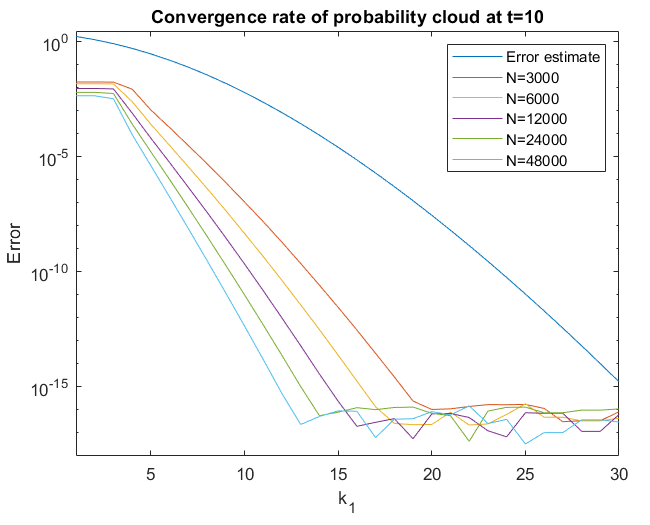} 
\includegraphics[scale=0.35]{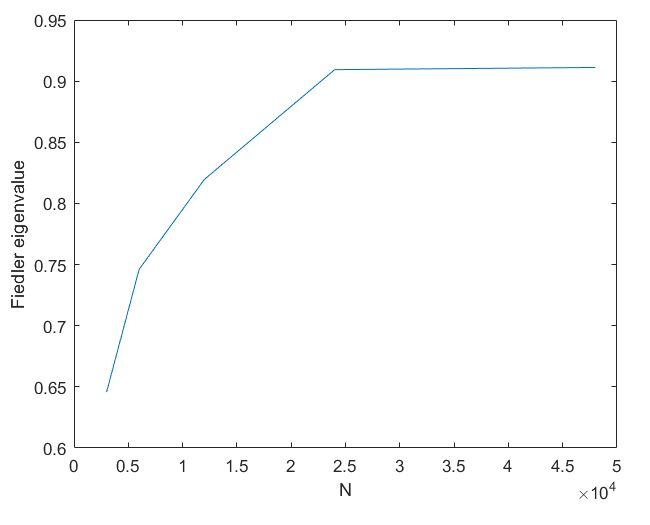} 
\caption{Absolute errors of Stage 1 Lanczos process in Algorithm~\ref{alg:twostagerom} versus 
the numbers of Lanczos steps $k_1$ for the family of random graph Laplacians (\ref{eqn:probLij}) (left). 
Fiedler eigenvalue for the family of random graph Laplacians (right).}
\label{fig:rand_graph_lanc_conv}
\end{figure}

We show in the left plot in Figure~\ref{fig:rand_graph_lanc_conv} how the $l_2$ error of global 
diffusion on the graph from a randomly chosen single node, a.k.a. probability cloud at some late time 
(here $t=10$) depends on $k_1$ at Stage 1 Lanczos process of Algorithm \ref{alg:twostagerom}. 
We also show a priori error bound obtained via Chebyshev series decomposition of the graph Laplacian
exponential \cite{druskin1989two}. This bound is uniform for all matrices with a given spectral interval 
(e.g., $[0,2]$ for normalized random walk graph-Laplacians), and it is tight for large matrices with spectrum 
densely distributed on the spectral interval.

Note that convergence of Lanczos process (as well as the error bound)
slows down monotonically with time, hence late times give the worse case scenario. 
As we observe in Figure~\ref{fig:rand_graph_lanc_conv}, the actual convergence rate is significantly 
faster than the bound, and for the family of graph Laplacians (\ref{eqn:probLij}) we do not benefit from 
Stage 2 of Algorithm \ref{alg:twostagerom}. Also, surprisingly, the error decays faster for larger 
graph Laplacians. This is caused by the increase of Fiedler eigenvalue with respect to the size of
graph Laplacian from the family (\ref{eqn:probLij}), as shown in the right plot in 
Figure~\ref{fig:rand_graph_lanc_conv}. Indeed, this eigenvalue determines the late time asymptotics 
of diffusion process on a graph, and the larger it is, the fewer steps Lanzcos process needs to converge.


\section{Interpretation in terms of finite-difference Gaussian rules}
\label{app:interpret}

To connect the clustering approaches presented here to the so-called 
finite-dif\-fer\-ence Gaussian rules, a.k.a optimal grids 
\cite{druskin1999gaussian,druskin2000gaussian,ingerman2000optimal,
asvadurov2000application,asvadurov2002application,borcea2005continuum}, 
we view the random-walk normalized graph-Laplacian $\bL_{RW}$
as a finite-difference approximation of the positive-semidefinite elliptic operator 
\begin{equation}
{\cal L} u(x) = - \frac{1}{\sigma(x)} \nabla \cdot \left[ \sigma(x) \nabla u(x) \right],
\end{equation} 
on a grid uniform in some sense defined on the data manifold, e.g., see \cite{Belkin2003LaplacianEF}. 
Note that since we assume the grid to be ``uniform'', all the variability of the weights of $\bL_{RW}$
is absorbed into the coefficient $\sigma(x) > 0$.

For simplicity, following the setting of \cite{borcea2014model}, let us consider the single 
input/single output (SISO) 1D diffusion problem on $x \in [0,1]$, $t \in (0, \infty)$:
\begin{equation}
u_t(x, t) - \frac{1}{\sigma(x)} [\sigma(x) u_x(x, t)]_x = 0, \quad 
u(x, 0) = \delta(x), 
\quad u_x(0, t) = 0, \quad u_x(1, t) = 0,
\label{eqn:1Ddiff}
\end{equation}
with a regular enough $\sigma(x) > 0$, and the diffusion transfer function defined as
\begin{equation}
F(t) = u(0, t).
\label{eqn:1DF}
\end{equation}
Since both input and output are concentrated at $x = 0$, the 
``target set'' consists of a single ``vertex'' corresponding to $x = 0$. Therefore, it does not make
sense to talk about clustering, however, we can still use the SISO dynamical system 
(\ref{eqn:1Ddiff})--(\ref{eqn:1DF}) to give a geometric interpretation of the embedding properties
of our reduced model and to provide the reasoning for Assumption~\ref{ass:z0}. 

The ROM (\ref{eqn:T})--(\ref{eqn:Dtilde}) constructed for the system (\ref{eqn:1Ddiff})
transforms it into
\begin{equation}
\tbu_t - \tD^{-1} \tL \tbu = 0, \quad 
\tbu|_{t=0} = \tD^{-1} \be_1,
\label{eqn:1DROM}
\end{equation}
where $\tbu, \be_1 \in \RR^{n}$, $\tD, \tL \in \RR^{n \times n}$, $n = k_2$,
with $\tD = \mbox{diag} \{ \widehat{h}_1 \widehat{\sigma}_1, \ldots, \widehat{h}_{n} \widehat{\sigma}_{n} \}$,
and $\tL$ is the second order finite-difference operator defined by
\begin{equation}
[\tL \tbu]_{i} = 
\frac{\sigma_{i}}{h_{i}}(\tu_{i} - \tu_{i-1}) - 
\frac{\sigma_{i+1}}{h_{i+1}}(\tu_{i+1} - \tu_{i}),
\qquad i = 1,\ldots,n,
\label{eqn:1DFD}
\end{equation} 
with $\tu_0$ and $\tu_{n+1}$ defined to satisfy the discrete Neumann boundary conditions
\begin{equation}
\tu_0 = \tu_1, \quad \tu_{n} = \tu_{n+1}.
\label{eqn:1DBC}
\end{equation}

As we expect from Sections~\ref{app:dblanczos} and \ref{sub:romgraph}, $\tL$ is indeed a 
tridiagonal matrix. Parameters $h_i, \widehat{h}_i$, $i = 1,2,\ldots,n$ can be interpreted as  
the steps of a primary and dual grids of a staggered finite-difference scheme, respectively, 
whereas $\sigma_i, \widehat{\sigma}_i$ are respectively the values of $\sigma(x) > 0$ at the 
primary and dual grid points. Assumption~\ref{ass:z0} then follows from the positivity of primary 
and dual grid steps (a.k.a. Stieljes parameters) given by the Stieljes theorem 
\cite{druskin1999gaussian}. 

\begin{figure}[htb]
\centering
\includegraphics[scale=0.45]{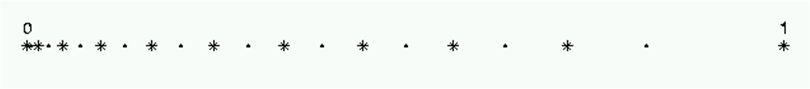} 
\caption{Finite-difference interpretation of ROM graph-Laplacian realization for the 
1D Laplace operator on $[0,1]$. Primary and dual grid nodes are dots and stars, respectively.}
\label{fig:grid}
\end{figure}

As an illustration, we display in Figure~\ref{fig:grid} the optimal grid with steps 
$h_i, \widehat{h}_i$, $i = 1,\ldots,10$ computed for $\sigma \equiv 1$. The continuum operator
${\cal L}$ is discretized on an equidistant grid on $[0,1]$ with $N = 100$ nodes, 
i.e., $\bL_{RW} \in \RR^{N \times N}$.
The optimal grid steps were computed using the gridding algorithm from \cite{druskin2016compressing}, 
which coincides with Algorithm~\ref{alg:twostagerom} and the subsequent transformations (\ref{eqn:Dtilde}) 
and (\ref{eqn:T}) for the 1D Laplacian operator. We observe that the grid is embedded in the domain $[0,1]$ 
and has a pronounced stretching away from the origin. We should point out, that such stretching is 
inconsistent with random-walk normalized graph-Laplacian formulation, employing uniform grids. 
Grid non-uniformity is the price to pay for the spectral convergence of the transfer function approximation. 
One can view Algorithm~\ref{alg:cluster1} as an embedding of the reduced-order graph back to the space 
of the original normalized random walk graph Laplacian $\bL_{RW}$. 
The randomized choice of sampling vertices  provides a uniform graph sampling. 

For the general MIMO problem, Proposition~\ref{prop:scale} yields a symmetric positive-semidefinite 
block-tridiagonal $\tL$ with zero sum of elements in every row. The classical definition of graph-Laplacians 
requires non-positivity of the off-diagonal entries, which may not hold for our $\tL$. However, it is known 
that operators with oscillating off-diagonal elements still allow for efficient clustering if the 
zero row sum condition remains valid \cite{Knyazev2017SignedLF}. Matrix $\tL$ still fits to a more 
general PDE setting, i.e., when the continuum Laplacian is discretized in an anisotropic media or 
using high order finite-difference schemes. Such schemes appear naturally when one wants to employ 
upscaling or grid coarsening in approximation of elliptic equations in multidimensional domains, 
which is how we can interpret the transformed ROM (\ref{eqn:T}). 

\section{Deflated block-Lanczos tridiagonalization process for sym\-met\-ric matrices}
\label{app:dblanczos}

Let $\bM = \bM^T \in \RR^{n \times n}$ be a symmetric matrix and $\bC \in \RR^{n \times m}$ be a 
``tall'' ($n > m$) matrix with orthonormal columns: $\bC^T \bC = \bI_m$, where $\bI_m$ 
is the $m \times m$ identity matrix. 
The conventional block-Lanczos algorithm successively constructs an orthonormal basis, the columns of
the orthogonal matrix $\widetilde{\bQ} \in \RR^{n \times m k}$, for the block Krylov subspace
\begin{equation}
{\K}_k(\bM, \bC) = \mbox{colspan} \{\bC, \bM \bC, \bM^2 \bC, \ldots, \bM^{k-1} \bC \},
\end{equation}
such that
\begin{equation} 
\widetilde{\bT} = \widetilde{\bQ}^T  \bM \widetilde{\bQ}
\end{equation}
is block tridiagonal and the first $m$ columns of $\widetilde{\bQ}$ are equal to $\bC$.
The deflation procedure allows to truncate the obtained basis at each step.

\begin{algorithm}[Deflated block-Lanczos process]
\label{alg:DBL}

\noindent \textbf{Input:} Symmetric matrix $\bM = \bM^T \in \RR^{n \times n}$, 
a matrix $\bC \in \RR^{n \times m}$ of initial vectors with orthonormal columns, 
maximum number of Lanczos steps $k$ such that $m k \le n$, 
and truncation tolerance $\varepsilon$.\\
\textbf{Output:} Deflated block tridiagonal matrix $\widetilde{\bT}$, 
orthogonal matrix $\widetilde{\bQ}$.

\textbf{Steps of the algorithm:}
\vspace{0.05in}
\begin{enumerate}
\item Set $\widetilde{\bQ}_1 = \bC$, $\bbeta_1 = \bI_m$, $m_1 = m$.
\item For $j = 1, 2, \ldots, k$:
\begin{enumerate}
\item Compute $\bR_j := \bM \widetilde{\bQ}_j$.
\item Compute $\balpha_j := \widetilde{\bQ}_j^T \bR_j$.
\item Compute $\bR_j: = \bR_j - \widetilde{\bQ}_j \balpha_j$. 
\item If $j > 1$ then set $\bR_j := \bR_j - \widetilde{\bQ}_{j-1} \bbeta_{j}^T$. 
\item Perform the SVD of $\bR_j$: 
\begin{equation}
\bR_j = \bU \bSigma \bW^T,
\end{equation}
with orthogonal $\bU \in \RR^{n \times m_j}$, $\bW \in \RR^{m_j \times m_j}$,
and a diagonal matrix of singular values $\bSigma$.
\item Truncate $\bU, \bSigma, \bW$ by discarding the singular vectors corresponding 
to the singular values less than $\varepsilon$. 
Denote the truncated matrices by 
$\widetilde{\bU} \in \RR^{n \times m_{j+1}}$, 
$\widetilde{\bSigma} \in \RR^{m_{j+1} \times m_{j+1}}$, 
$\widetilde{\bW} \in \RR^{m_{j} \times m_{j+1}}$ , 
where $m_{j+1} \le m_j$ is the number of the remaining, non-truncated singular modes.
\item If $m_{j+1} = 0$ then exit the for loop.
\item Set $\widetilde{\bQ}_{j+1} := \widetilde{\bU} \in \RR^{n \times m_{j+1}} $.
\item \label{step:beta} Set 
$\bbeta_{j+1} := \widetilde{\bSigma} \widetilde{\bW}^T \in \RR^{m_{j+1} \times m_j} $.
\item \label{step:reorth} Perform reorthogonalization 
$\widetilde{\bQ}_{j+1} := \widetilde{\bQ}_{j+1} - 
\sum\limits_{i=1}^j \widetilde{\bQ}_i (\widetilde{\bQ}^T_i \widetilde{\bQ}_{j+1})$
if needed.
\end{enumerate}
\item endfor
\item Let $\widetilde{k}$ be the number of performed steps, set
\begin{equation}
\widetilde{\bQ} := 
\left[ \widetilde{\bQ}_1, \widetilde{\bQ}_2, \ldots, \widetilde{\bQ}_{\widetilde{k}} \right]
\in \RR^{n \times \widetilde{n}},
\end{equation}
where $\widetilde{n} = \sum\limits_{j=1}^{\widetilde{k}} m_j.$
\item Set 
\begin{equation}
\widetilde{\bT} = 
\begin{bmatrix}
\balpha_1 & \bbeta_2^T  &          &  &  \\
\bbeta_2  & \balpha_2 & \bbeta_3^T &  &  \\
          & \ddots    & \ddots   & \ddots &  \\
        &           & \bbeta_{\widetilde{k}-1} & \balpha_{\widetilde{k}-1} & \bbeta_{\widetilde{k}}^T \\
        &           &                & \bbeta_{\widetilde{k}}  & \balpha_{\widetilde{k}}
\end{bmatrix} \in \RR^{\widetilde{n} \times \widetilde{n}}.
\label{eqn:Ttilde}
\end{equation}
\end{enumerate}
\end{algorithm}

Note that step (\ref{step:reorth}) of Algorithm \ref{alg:DBL} is computationally 
expensive and is only needed for computations with finite precision. In practice,
it is infeasible to perform for large data sets in the first stage of Algorihtm~\ref{alg:twostagerom}.
However, in the second stage of Algorithm~\ref{alg:twostagerom} and also when performing 
the third block-Lanczos process in ROGL construction, we deal with relatively small matrices 
$\bT_1 \in \RR^{n_1 \times n_1}$, $\bT_2 \in \RR^{n \times n}$, $n_1,n \ll N$, so 
the reorthogonalization in step (\ref{step:reorth}) becomes computationally feasible.

Given that the bulk of the computational effort of ROGL construction via
Algorithm~\ref{alg:twostagerom} is spent in Algorithm~\ref{alg:DBL}, one may consider its
parallelization to boost the overall performance. The most straightforward way of doing so is 
to parallelize the matrix product computation in step (2a). For a system with $p$ processors,
one may store $N/p$ rows of $\bM$ on each processor, multiply these submatrices by 
$\widetilde{\bQ}_j$ in parallel and then communicate the results across the processors 
so that each one has access to its own copy of $\bR_j = \bM \widetilde{\bQ}_j$. 
All other operations of Algorithm~\ref{alg:DBL} can be performed locally at each processor
to avoid communicating anything else except for the rows of $\bR_j$.

\section{Proof of Proposition \ref{prop:distpres}}
\label{app:distpres}

\begin{proof}
Similar to (\ref{eqn:diffdist}) and (\ref{eqn:ctdist}), 
for the vertex set $\tG$ of the reduced-order graph we have 
\begin{equation}
\left( D^p_{jk}(\tG) \right)^2 = 
\left( \sqrt{\widetilde{D}_{jj}} \be_j^T - \sqrt{\widetilde{D}_{kk}} \be_k^T \right)
( \bI - \bT_3)^{2p}
\left( \sqrt{\widetilde{D}_{jj}} \be_j - \sqrt{\widetilde{D}_{kk}} \be_k \right)
\label{eqn:diffdist_rom}
\end{equation}
and
\begin{equation}
C^2_{jk}(\tG) = 
\left( \frac{1}{\sqrt{\widetilde{D}_{jj}}} \be_j^T - \frac{1}{\sqrt{\widetilde{D}_{kk}}} \be_k^T \right)
\bT^\dagger_3
\left( \frac{1}{\sqrt{\widetilde{D}_{jj}}} \be_j - \frac{1}{\sqrt{\widetilde{D}_{kk}}} \be_k \right).
\label{eqn:ctdist_rom}
\end{equation}
Here in a slight abuse of notation we let $\be_j$ and $\be_k$ be unit vectors in $\RR^{n}$. 

Due to (\ref{eqn:diagmatch}), for $\bD$ given by (\ref{eqn:randwalkdiag}) we obtain
\begin{equation}
\left( D^p_{i_j, i_k}(G) \right)^2 - \left( D^p_{jk}(\tG) \right)^2 = 
L_{i_j, i_j} \Delta P^{2p}_{jj} + L_{i_k, i_k} \Delta P^{2p}_{kk} 
- 2 \sqrt{L_{i_j, i_j} L_{i_k, i_k}} \Delta P^{2p}_{jk}
\end{equation}
and
\begin{equation}
C^2_{i_j, i_k}(G) - C^2_{jk}(\tG) = 
\frac{1}{L_{i_j, i_j}} \Delta J_{jj} + \frac{1}{L_{i_k, i_k}} \Delta J_{kk}
- 2 \frac{1}{\sqrt{L_{i_j, i_j} L_{i_k, i_k}}} \Delta J_{jk}
\end{equation}
where 
\begin{equation}
\Delta P^{2p}_{jk} = 
\be_{i_j}^T ( \bI - \bA)^{2p} \be_{i_k} - \be_{j}^T ( \bI - \tA )^{2p} \be_{k}, \quad
j, k = 1,\ldots,m,
\end{equation} 
and 
\begin{equation}
\Delta J_{jk} = \be_{i_j}^T \bA^\dagger \be_{i_k} - \be_j^T \tA^\dagger \be_k, \quad
j, k = 1,\ldots,m,
\end{equation} 
are ROM errors of approximations of polynomials and the pseudo-inverse, respectively. 

Since the ROGL is obtained via a three-stage process, we make the analysis more explicit by splitting 
the errors into three parts corresponding to each stage
\begin{equation}
\Delta P^{2p}_{jk} = \Delta^1 P^{2p}_{jk} + \Delta^2 P^{2p}_{jk} + \Delta^3 P^{2p}_{jk},
\end{equation}
where 
\begin{align*}
\Delta^1 P^{2p}_{jk} & = \be_{i_j}^T( \bI - \bA)^{2p} \be_{i_k} - \be_j^T (\bI - \bT_1)^{2p} \be_k, \\
\Delta^2 P^{2p}_{jk} & = \be_j^T (\bI - \bT_1)^{2p} \be_k - \be_j^T (\bI - \tA_{12})^{2p} \be_k, \\
\Delta^3 P^{2p}_{jk} & = \be_j^T (\bI - \tA_{12})^{2p} \be_k - \be_j^T (\bI - \tA)^{2p} \be_k.
\end{align*}
Similarly, for the pseudo-inverse error we have
\begin{equation}
\Delta J_{jk} = \Delta^1 J_{jk} + \Delta^2 J_{jk} + \Delta^3 J_{jk},
\end{equation} 
where 
\begin{align*}
\Delta^1 J_{jk} & = \be_{i_j}^T \bA^\dagger \be_{i_k} - \be_j^T \bT^\dagger_1 \be_k, \\
\Delta^2 J_{jk} & = \be_j^T \bT^\dagger_1 \be_k - \be_j^T \tA^\dagger_{12} \be_k, \\
\Delta^3 J_{jk} & = \be_j^T \tA^\dagger_{12} \be_k - \be_j^T \tA^\dagger \be_k.
\end{align*}

We note that $\Delta^3 P^{2p}_{jk} = \Delta^3 J_{jk} = 0$ since the transformation from 
$\tA_{12}$ to $\tA$ is unitary. To finalize the proof, we refer to the known results in the 
theory of model reduction via Krylov subspace projection \cite{druskin1995krylov}. In particular, 
$\Delta^2 P^{2p}_{jk} \rightarrow 0$ and $\Delta^1 J_{jk} \rightarrow 0$ exponentially
in $n$. Also, $\Delta^2 J_{jk}=0$ for $k_2>1$ and $\Delta^1 P^{2p}_{jk} = 0$ for $k_1 \ge p$.
\end{proof}

It follows from the proof that if the second stage of Algorithm~\ref{alg:twostagerom} is exact 
then $D^p_{i_j, i_k}(G) = D^p_{jk} (\tG)$ for $k_1 \ge p$ (see \cite{druskin1995krylov}).

\bibliography{graphbib}

\end{document}